\documentclass[final,1p,times]{elsarticle}

\usepackage[T1]{fontenc}        % 可选：如无字体问题可删
\usepackage{amsmath,amssymb,amsfonts}
\usepackage{amsthm}
\usepackage{booktabs}
\usepackage{multirow}
\usepackage{makecell}
\usepackage{threeparttable}
\usepackage{tabularx}
\usepackage{array}
\usepackage{arydshln}
\usepackage{nicefrac}
\usepackage{microtype}
\usepackage{xcolor}
\usepackage{enumitem}
\usepackage{romannum}
\usepackage{subcaption}
\usepackage[ruled,vlined]{algorithm2e}
\SetKwInput{KwInit}{Initialize}

\usepackage[hidelinks]{hyperref}
\usepackage{url}

\newtheorem{theorem}{Theorem}

\newtheorem{proposition}{Proposition}
\newtheorem{assumption}{Assumption}

\newcommand{\algc}{Return-Cost Regularized Constrained Decision Transformer}
\newcommand{\algcb}{RCDT}
\newcommand{\algw}{Trajectory-weighted Value-regularized Constrained Decision Transformer}
\newcommand{\algwb}{TVCDT}

\newcommand{\abbc}{CR2DT}

\usepackage{amssymb}
\usepackage{amsmath}
\journal{MDPI}

\begin{document}

\begin{frontmatter}

\title{Conditional Sequence Modeling for Safe Reinforcement Learning}


\author[label1]{Wensong Bai}
\ead{wensongb@zju.edu.cn}

\author[label1,label2]{Chao Zhang\corref{cor1}}
\ead{zczju@zju.edu.cn}

\author[label1]{Qihang Xu}
\ead{22351233@zju.edu.cn}

\author[label1]{Chufan Chen}
\ead{chufan.chen@zju.edu.cn}

% \author[label1]{Hanbin Zhao}
% \ead{zhaohanbin@zju.edu.cn}

\author[label1]{Chenhao Zhou}
\ead{zhouchenhao}

\author[label1]{Hui Qian}
\ead{qianhui@zju.edu.cn}

\cortext[cor1]{Corresponding author. Email: zczju@zju.edu.cn}

\affiliation[label1]{
  organization={College of Computer Science and Technology, Zhejiang University},
  city={Hangzhou},
  country={China}
}

\affiliation[label2]{
  organization={Advanced Technology Institute, Zhejiang University},
  city={Hangzhou},
  country={China}
}

%% Abstract
\begin{abstract}
Offline safe reinforcement learning (RL) aims to learn policies from a fixed dataset while maximizing performance under cumulative cost constraints. In practice, deployment requirements often vary across scenarios, necessitating a single policy that can adapt zero-shot to different cost thresholds. However, most existing offline safe RL methods are trained under a pre-specified threshold, yielding policies with limited generalization and deployment flexibility across cost thresholds.
Motivated by recent progress in conditional sequence modeling (CSM), which enables flexible goal-conditioned control by specifying target returns, we propose \algcb, a CSM-based method that supports zero-shot deployment across multiple cost thresholds within a single trained policy. \algcb~is the first CSM-based offline safe RL algorithm that integrates a Lagrangian-style cost penalty with an auto-adaptive penalty coefficient. To avoid overly conservative behavior and achieve a more favorable return--cost trade-off, a reward--cost-aware trajectory reweighting mechanism and Q-value regularization are further incorporated.
Extensive experiments on the DSRL benchmark demonstrate that \algcb~consistently improves return--cost trade-offs over representative baselines, advancing the state-of-the-art in offline safe RL.
\end{abstract}


\end{frontmatter}

\section{Introduction} \label{sec:intro}

Offline safe reinforcement learning (RL) aims to learn reliable decision-making policies from a fixed dataset while maximizing task performance under cumulative cost constraints, without any additional environment interaction \cite{koiralalatent,gong2025offline,suboundary,guo2025constraint}. This setting is particularly relevant for safety-critical domains such as robotics and autonomous systems, where online exploration can be prohibitively expensive and unsafe\cite{zhang2025wococo,ji2025exbody2}. In practice, safety requirements are rarely unique: different deployment scenarios or operating conditions may impose different cost constraints, and policies often need to evaluate and tune the return--cost trade-off across a range of cumulative cost thresholds \cite{liu2023constrained,gong2025offline}. Such variability highlights the need for policies with \emph{zero-shot} adaptability across multiple cost thresholds \cite{kirk2023survey}. However, most prior offline safe RL methods are trained under a pre-specified cost threshold, resulting in policies that are tied to that specific constraint and hence do not naturally support one-policy, zero-shot deployment across varying thresholds \cite{leecoptidice,xu2022constraints,zhengsafe}.

Recent advances in offline RL via conditional sequence modeling (CSM) suggest a promising route to accommodate varying deployment requirements within a single policy \cite{wu2024elastic,hu2024q,zheng2024decomposed,zhengdecision,hu2024harmodt}. Conceptually, CSM reformulates offline RL into supervised learning on trajectory data by training a goal-conditioned action predictor \cite{chen2021decision,brandfonbrener2022does}. One representative CSM approach is the Decision Transformer (DT)~\cite{chen2021decision}, which trains a Transformer to predict actions from trajectory prefixes conditioned on a target return-to-go (RTG) signal.
At inference time, DT treats decision making as a conditional generation problem, and thus can adjust its behavior by varying the target RTG token. By making decisions directly from temporally extended trajectories, this paradigm also bypasses recursive value estimation, thereby mitigating bootstrapping errors and challenges associated with dynamic programming in offline settings~\cite{chen2021decision,brandfonbrener2022does}. These properties make CSM a compelling backbone candidate for offline safe RL, where one would ideally like to adapt the return-cost trade-off across budgets without retraining.

Nevertheless, a straightforward extension of DT to constrained settings, by introducing a cost-to-go (CTG) token and conditioning on both return and cost, does not necessarily yield satisfactory safety behavior. The key issue is that maximum-likelihood training encourages the model to match the conditioning signals, which obscures the asymmetry in constrained Markov decision processes (CMDPs) where return is to be optimized while cost specifies a constraint. As a result, naive CTG matching can induce unstable return-cost trade-offs and even constraint violations, especially when the dataset provides limited coverage of low-cost behaviors. Such instability and constraint violations under naive CTG-conditioned DT variants have been widely observed in prior work and in our experiments \cite{koiralalatent,suboundary,liu2023constrained,guo2025constraint}. Consequently, a central question arises: \emph{how can we endow CSM-based policies with a principled and robust mechanism for achieving favorable return-cost trade-offs while preserving their zero-shot adaptation flexibility?}

To develop a robust CSM-based approach for offline safe RL, we first analyze the discrepancy between the specified RTG/CTG signals and the resulting expected return and cumulative cost induced by a conditioned CSM policy.
More precisely, we consider a \emph{fixed} conditioning function $F(\cdot)$ that determines the RTG/CTG inputs used at inference, and let $R(\tau)$ and $C(\tau)$ denote the (random) return and cumulative cost of a trajectory $\tau$ generated by the behavior policy $\beta$ from an initial state $s_1$. 
Our analysis shows that, under mild CMDP assumptions, this discrepancy admits an upper bound that scales as $O(\varepsilon(\tfrac{1}{\alpha_F}+2) H^2)$, where $\alpha_F$ lower-bounds the joint return--cost coverage of the target profile in the offline dataset, i.e.,
$P_{\pi_{\beta}}\bigl((R(\tau),C(\tau)) = F(s_1)|s_1\bigr)\ge \alpha_F$.
Consequently, when the desired return--cost conditioning function $F(\cdot)$ has low coverage (small $\alpha_F$), the induced optimal policy may substantially deviate from the intended RTG/CTG targets, leading to unstable return--cost trade-offs and potential constraint violations. 
In particular, simply specifying a high RTG together with a low CTG does not reliably yield favorable return--cost trade-offs on challenging datasets where safe high-return behaviors are under-represented. 
These analysis motivates augmenting naive conditioning with \emph{weighted} and \emph{regularized} training mechanisms that shift learning toward trajectories with better return--cost trade-offs while preserving zero-shot adaptability.

Building on this insight, we propose the first CSM-based offline safe RL algorithm that integrates a Lagrangian-style cost penalty while preserving zero-shot adaptation across different cost thresholds. We term the proposed algorithm \algc~(\algcb). 
At a high level, \algcb~minimizes a weighted maximum-likelihood objective augmented with an explicit cost penalty, where the penalty term is induced from a primal--dual Lagrangian formulation of CMDPs. Specifically, \algcb~treats the cost-penalty coefficient as a dual variable and updates it via a standard dual-ascent step driven by the current cost estimate, which provides a principled, auto-adaptive mechanism for tuning the regularization strength while avoiding hard-wiring the policy to a single cost threshold.
To avoid overly conservative behavior and achieve a more favorable return--cost trade-off, \algcb~further incorporates two complementary ingredients: (\romannumeral1) a reward--cost-aware trajectory-level reweighting mechanism that emphasizes trajectories with desirable return--cost profiles; and (\romannumeral2) Q-value regularization as an additional optimization signal to steer the policy toward high-return regions. 
Together, these designs decouple policy learning from any single pre-specified cost threshold while maintaining strong task performance, thereby enabling one-policy deployment across multiple cost thresholds in a zero-shot manner.

Our contributions are summarized as follows:
\begin{enumerate}
    \item 
    We propose \algcb, the first CSM-based offline safe RL method that integrates a Lagrangian-style cost penalty with an \emph{auto-adaptive} dual-ascent update of the penalty coefficient. 
    \algcb~achieves favorable return--cost trade-offs while enabling zero-shot deployment across multiple cost thresholds, and empirically outperforms state-of-the-art offline safe RL baselines under stringent constraints.
    \item 
    We theoretically characterize the discrepancy between the prescribed RTG/CTG signals and the realized expected return and cumulative cost of CSM policies through the lens of \emph{joint return--cost coverage}. 
    The theoretical analysis pinpoints when a CSM policy can be regarded as a reliable backbone for offline safe RL and establishes performance limits for zero-shot control via RTG/CTG conditioning.
    \item
    To encourage strong return performance without inducing overly conservative behavior, we introduce a reward--cost-aware trajectory-level reweighting scheme together with Q-value regularization as complementary training signals. This design introduces an additional inductive bias toward trajectories with favorable return--cost trade-offs. Moreover, we show that the commonly used expert KL regularizer can be interpreted as a special case of our reweighting mechanism.
\end{enumerate}
We evaluate \algcb~on the DSRL benchmark~\cite{liu2023datasets} under multiple cumulative cost thresholds, where a single trained policy is assessed in a zero-shot manner across all thresholds. 
Experimental results show consistent improvements in return--cost trade-offs over representative baselines, establishing \algcb~as a strong candidate for advancing the state-of-the-art in offline safe RL.

\section{Related Work}

Offline RL algorithms learn a policy entirely from a static offline dataset of past interactions, without additional access to interactions with environment during training~\cite{kumar2020conservative,fujimoto2021minimalist,hu2024q,hu2025graph}. This paradigm is particularly useful when online interaction is expensive or high-risk, for example, roboric control systems and safety-critical applications.

\subsection{Offline Safe Reinforcement Learning}\label{sec:related1}

Offline safe RL aims to learn reliable decision-making policies from a fixed dataset while maximizing task performance under cumulative cost constraints, without any additional environment interaction. Compared with standard offline RL, the central challenge is to ensure reliable constraint satisfaction under distributional shift in the state-action space, rather than merely reducing empirical costs on the training data.\cite{liu2023datasets}

Classical offline safe RL methods largely build on imitation learning or Lagrangian extensions of offline value-based RL algorithms. Behavior cloning (BC) treats policy learning as supervised learning on a static dataset, directly fitting the conditional action distribution given states. Safety-aware variants often introduce simple data selection or reweighting strategies that emphasize trajectories with lower cumulative cost, so that the learned policy imitates predominantly safe behaviors; however, such filtering may reduce data coverage and induce overly conservative policies \cite{liu2023datasets}. Another class of approaches augments standard offline value-based RL algorithms with a Lagrangian relaxation of the CMDP constraint. Representative examples include BCQ-Lag and BEAR-Lag, which extend BCQ~\cite{fujimoto2019off} and BEAR~\cite{kumar2019stabilizing} by jointly learning reward and cost critics together with a Lagrange multiplier, and CPQ~\cite{xu2022constraints}, which incorporates conservative regularization on both reward and cost critics to keep the learned policy close to the behavior distribution. COptiDICE~\cite{lee2022coptidice} instead operates in the occupancy-measure space, estimating distribution corrections and solving a convex objective that maximizes return under an explicit upper bound on the expected cost.

Another line of research improves safety by explicitly reasoning about data quality and feasible regions \cite{koiralalatent,zhengsafe,gong2025offline,yao2024oasis}.
LSPC~\cite{koiralalatent} learns latent safety constraints using a conditional variational autoencoder and restricts policy optimization to a latent region that jointly encodes reward and cost signals, thereby mitigating out-of-distribution actions in the original state-action space. FISOR~\cite{zhengsafe} revisits offline safe RL from a reachability perspective, translating hard safety constraints into the identification of the largest feasible region supported by the dataset and then solving a feasibility-dependent objective, yielding a policy that can be implemented via weighted behavior cloning with a diffusion model. TraC~\cite{gong2025offline} instead recasts offline safe RL as a trajectory classification problem: it constructs desirable and undesirable trajectory sets based on cumulative rewards and costs, and optimizes a policy to imitate desirable trajectories while avoiding undesirable ones, effectively performing trajectory-level filtering. The diffusion-based offline safe RL method OASIS~\cite{yao2024oasis} employs a conditional diffusion model to reshape the offline action distribution toward trajectories that jointly exhibit low cost and high return, thereby suppressing unsafe behaviors that are poorly supported by the dataset.

Motivated by recent advances in large-scale sequence modeling \cite{chen2021decision,hu2024q,bairebalancing}, conditional sequence modeling (CSM) has emerged as a promising paradigm for offline safe RL \cite{liu2023constrained,suboundary}. Constrained Decision Transformer (CDT)~\cite{liu2023constrained} extends Decision Transformer by conditioning action prediction on both return-to-go (RTG) and cost-to-go (CTG) tokens, thus formulates offline safe RL as a multi-objective sequence modeling problem in which a single CSM policy can be evaluated under different cost thresholds. B2R~\cite{suboundary} further refines how cost information is injected into the sequence model, encouraging the learned policy to respect safety boundaries while maintaining competitive returns. 
Compared with imitation-learning heuristics, Lagrangian extensions of value-based offline RL, and feasible-region based approaches, these CSM methods provide a particularly direct mechanism for trading off performance and safety through target conditioning, making one-policy, multi-threshold deployment more natural.
Nevertheless, existing CSM-based methods largely treat RTG/CTG conditioning as a heuristic control interface, and the fundamental limits of such zero-shot adaptation remain underexplored.
In this work, we analyze the discrepancy between the specified RTG/CTG signals and the resulting expected return and cumulative cost induced by a conditioned CSM policy through the lens of joint return--cost coverage, thereby characterizing when target conditioning can be reliably realized from offline data.
Building on this perspective, we propose \algcb, which augments CSM training with weighted and regularized mechanisms, yielding improved return--cost trade-offs and stronger zero-shot performance over prior CSM-based offline safe RL methods.

\subsection{Conditional Sequence Modeling for Offline RL}\label{sec:related2}

Casting RL as conditional sequence modeling has emerged as a powerful alternative to value-based and policy-gradient methods in offline settings \cite{chen2021decision,hu2024q,zeng2023goal,bairebalancing}. Instead of explicitly estimating value functions or behavior-regularized policies, CSM methods train autoregressive models on offline trajectories and condition them on side information such as desired returns, goals, or future states \cite{chen2021decision,zeng2023goal,kim2024stitching}. This formulation enables the direct reuse of advances in sequence modeling and large-scale supervised training, while bypassing several instability sources that commonly arise in traditional offline RL, most notably the reliance on bootstrapping to propagate value estimates (one of the “deadly triad” \cite{sutton1998reinforcement}), and the myopic behaviors induced by discounting future rewards~\cite{chen2021decision,zhou2023free}.

The most representative instantiation of this idea is the Decision Transformer (DT)~\cite{chen2021decision}, which conditions a causal Transformer \cite{vaswani2017attention} on past states, actions, and a target return, and predicts future actions so as to achieve the specified return. Subsequent works refine this return-conditioned formulation by modifying the conditioning signal or the learning objective. For instance, Q-learning Decision Transformer \cite{yamagata2023q} and its variants replace the raw return with value-based targets obtained via dynamic programming (DP), aiming to better capture long-horizon optimality and alleviate sparse-reward issues~\cite{yamagata2023q,gao2024act}. Advantage-guided and return-regularized methods reshape the supervised loss using critic estimates to emphasize near-optimal actions in the offline dataset~\cite{weiadvantage,tu2025dataset}. Similarly, Q-value Regularized Transformers (QT) \cite{hu2024q} and Critic-guided Decision Transformers (CGDT) \cite{wang2024critic} introduce estimated Q-values as explicit regularizers of the negative log-likelihood loss, sharpening the model's preference for high-quality behaviors. 

Another line of work employs sequence models as trajectory-level dynamics models rather than purely as policies. The Trajectory Transformer learns a joint model over sequences of states, actions, and rewards, and combines it with beam-search-style planning to generate high-return trajectories from offline data~\cite{janner2021offline}. Bootstrapped variants further improve upon this idea by leveraging the learned sequence model to self-generate additional synthetic trajectories and augment the offline dataset~\cite{wang2022bootstrapped}. Complementary to these data-augmentation approaches, multimodal architectures such as the Decision Transducer disentangle states, actions, and rewards into separate unimodal sequences and selectively fuse them, improving sample efficiency and representation flexibility for offline RL~\cite{wang2023trajectory}. These approaches highlight that sequence modeling can serve as a versatile backbone for both policy learning and model-based planning in offline RL. 

CSM has also been extended to safety-constrained settings by incorporating cost signals or richer specifications into the conditioning interface, enabling a single model to represent different safety--performance preferences~\cite{liu2023constrained,wang2024safe,guo2024temporal}. We refer readers to Section~\ref{sec:related1} for a detailed discussion of representative CSM-based offline safe RL methods.

\section{Preliminary}

\subsection{Constrained Markov Decision Process}
Safe reinforcement learning can be naturally formalized as a Constrained Markov Decision Process (CMDP) \cite{altman2021constrained}.
A finite horizon CMDP is defined as a tuple $(\mathcal{S}, \mathcal{A}, \mathcal{T}, r, c, \mu, H)$,
where $\mathcal{S}$ and $\mathcal{A}$ denote the (continuous) state and action spaces.
$r: \mathcal{S} \times \mathcal{A} \rightarrow \mathbb{R}$ is the reward function, and $c: \mathcal{S} \times \mathcal{A} \rightarrow [0, c_{max}]$ is the cost function.
$\mathcal{T}$ defines the transition dynamics, i.e., $s_{t+1} \sim \mathcal{T}(\cdot|s_t,a_t)$.
$\mu$ is the initial state distribution, i.e., $s_{1} \sim \mu$, and $H$ is the time horizon. The return and cumulative cost of a trajectory $\tau= \{(s_t, a_t, r_t, c_t)_{t=1}^H\}$ are denoted by $R(\tau) = \sum_{t=1}^{H} r_t$ and $C(\tau) = \sum_{t=1}^{H} c_t$, respectively. A policy $\pi$ is considered feasible if its expected cumulative cost does not exceed a threshold $\kappa \in [0, +\infty)$, i.e. $\mathbb{E}_{\tau \sim \pi}C(\tau) \leq \kappa$. 

In the offline safe RL setup, algorithms are given a static dataset $\mathcal{D}$ consisting of multiple trajectories, collected by a behavior policy $\pi_{\beta}$. The goal of offline safe RL is to learn a feasible policy $\pi$ from $\mathcal{D}$, which maximizes the expected return while satisfying a cost constraint:
\begin{equation}\label{eq:goal-offRL}
    \text{arg}\max_{\pi} \mathbb{E}_{\tau \sim \pi}[R(\tau)] \quad \text{subject to} \quad \mathbb{E}_{\tau \sim \pi}[C(\tau)] \leq \kappa.
\end{equation}

\subsection{Conditioned Sequence Modeling}
Motivated by the wide-ranging applications and remarkable success of transformer models in various domains \cite{vaswani2017attention, brown2020language, chen2021decision, hulearning,hu2023prompt,kongmastering}, CSM-based methods formulate policy learning as a conditional supervised learning problem. A representative CSM approach for offline safe RL is the \emph{constrained decision transformer} (CDT)~\cite{liu2023constrained}, a safety-aware extension of DT~\cite{chen2021decision} in CMDP settings. The model architecture of CDT is shown in Figure \ref{fig:crvdt}. Specifically, during training, the policy $\pi$ is optimized by minimizing the empirical negative log-likelihood (NLL) loss over a dataset $\mathcal{D}$ of trajectories (we omit normalization here):
\begin{equation}\label{eq: NLL} 
    \mathcal{L}_{\text{CDT}}(\pi) = - {\textstyle\sum_{\tau \in \mathcal{D}}} {\textstyle\sum_{1 \leq t \leq H}} \log \pi(a_t \mid s_t, \bar{R}_t, \bar{C}_t,\bar{\tau}_{t-1}^{K}), 
\end{equation}
where $\bar{R}_t := \sum_{i=t}^{H} r_i$ is the \emph{return-to-go} (RTG) and $\bar{C}_t := \sum_{i=t}^{H} c_i$ is the \emph{cost-to-go} (CTG) for $\tau$ at time step $t$. $\bar{\tau}_{t-1}^{K} = (\bar{R}_{t-K}, \bar{C}_{t-K}, s_{t-K}, a_{t-K}, \dots, \bar{R}_{t-1}, \bar{C}_{t-1}, s_{t-1}, a_{t-1}) $ denotes the length-$K$ historical context.\footnote{When \(t-K < 1\), the context window is truncated to the available prefix of the trajectory.} During inference, the model predicts a sequence of actions autoregressively conditioned on target RTG and CTG tokens. 

Compared with Lagrangian extensions of value-based methods, CSM-based methods offer several advantages in the offline safe RL setting \cite{liu2023constrained,suboundary,wang2024critic,hu2024q}. First, by conditioning on both RTG and CTG signals, a single CSM policy can represent a continuum of return--cost trade-offs and be evaluated under different cost thresholds without retraining \cite{liu2023constrained,suboundary}. Moreover, recasting policy learning as a supervised sequence prediction problem enables the use of stable and scalable transformer architectures that can fully exploit long-horizon dependencies and temporal structure that are critical for safe decision-making \cite{chen2021decision,liu2023constrained}.

\begin{figure}[t]
    \centering
    \includegraphics[width=0.9\linewidth]{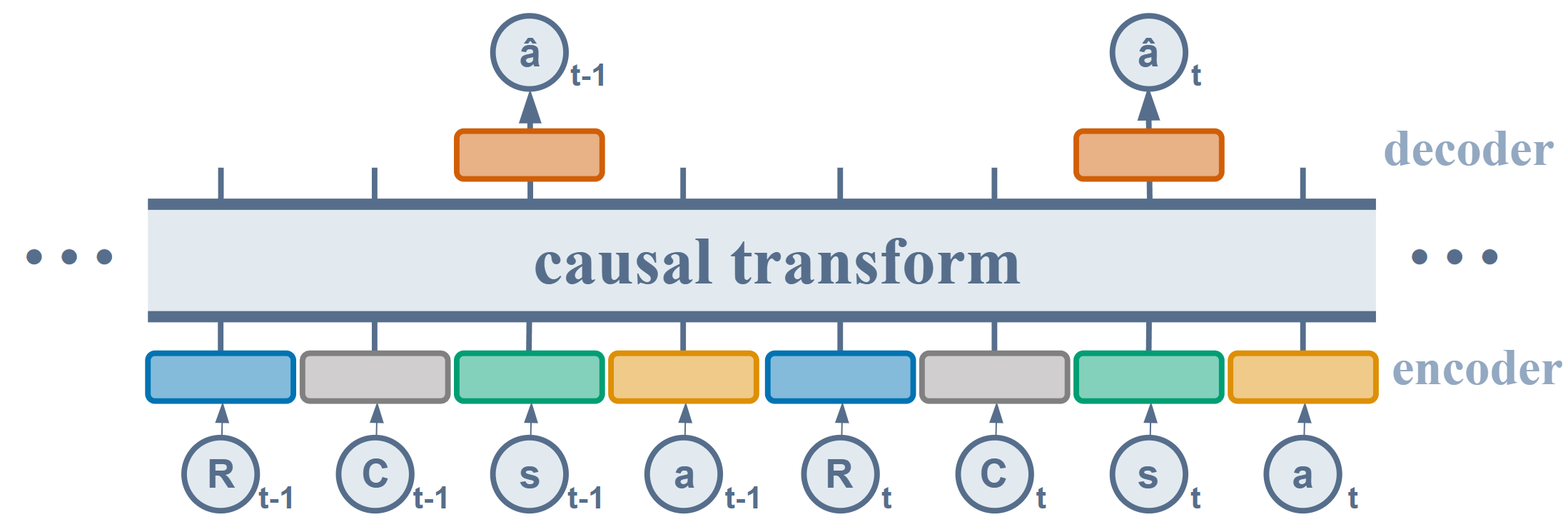}
    \caption{Overview of constrained decision transformer architecture.}
    \label{fig:crvdt}
\end{figure}

\section{Methodology} \label{sec:Methodology-new}

\subsection{Conditioned Sequence Modeling in CMDPs}
\label{subsec:cmdp-theory}

We analyze how conditional sequence modeling behaves in an idealized CMDP setting with infinite data and a fully expressive policy class. Our goal is to show that, under suitable assumptions, the policy that optimizes the CDT objective~\eqref{eq: NLL} is guaranteed to be near-optimal for the underlying CMDP. Let $P_{\pi_{\beta}}$ denote the joint distribution over states, actions, returns, and costs induced by $\pi_{\beta}$ and the CMDP dynamics. Recall that CDT is trained to maximize the likelihood of actions conditioned on states and target return-cost signals. Thus, when the model is given a conditioning function,
$$
F : \mathcal{S} \to \mathbb{R}^2, \quad \text{and} \quad
F(s) = \bigl(F_R(s), F_C(s)\bigr),
$$
the CDT loss in~\eqref{eq: NLL} attempts to learn the conditional distribution
$P_{\pi_{\beta}}(a \mid s, F(s))$.
In the infinite-data and realizable limit, the optimal CDT policy (denoted by $\pi^{\mathrm{CDT}}_F$) for a fixed conditioning function $F$ can be written as:
\begin{equation}
\pi^{\mathrm{CDT}}_F(a \mid s)
= P_{\pi_{\beta}}\bigl(a \mid s, F(s)\bigr)
= \frac{P_{\pi_{\beta}}(a \mid s)\, P_{\pi_{\beta}}\bigl(F(s) \mid s,a\bigr)}
       {P_{\pi_{\beta}}\bigl(F(s) \mid s\bigr)}
= \pi_{\beta}(a \mid s)\,
  \frac{P_{\pi_{\beta}}\bigl(F(s) \mid s,a\bigr)}
       {P_{\pi_{\beta}}\bigl(F(s) \mid s\bigr)}.
\label{eq:cdt-optimal-policy}
\end{equation}
Here, $F_R(s)$ and $F_C(s)$ can be interpreted as the target RTG and CTG associated with state $s$, and
$P_{\pi_{\beta}}\bigl(F(s) \mid s,a\bigr)$
denotes the probability (under $\pi_{\beta}$) that the trajectory starting from $(s_1 = s, a_1 = a)$ attains total return and cost equal to $F(s)$.
In other words, $\pi^{\mathrm{CDT}}_F$ re-weights the behavior policy $\pi_{\beta}(\cdot \mid s)$ according to how likely each action is to lead to a trajectory whose joint return-cost pair matches the conditioning signal $F(s)$.

Our analysis relies on a coverage condition and a mild regularity assumption.
For an initial state $s_1 \in \mathcal{S}$, let $R(\tau)$ and $C(\tau)$ denote the (random) return and cost of a trajectory $\tau$ sampled from the behavior policy $\pi_{\beta}$ starting from $s_1$.
Represent the return-cost coverage for $s_1$ by
$
P_{\pi_{\beta}}\bigl( (R(\tau), C(\tau)) = F(s_1) | s_1 \bigr),
$
which is the conditional probability that the total return and cost of the trajectory match the conditioning signal $F(s_1)$.
For any policy $\pi$, we define
$J_R(\pi) := \mathbb{E}_{\tau \sim \pi}[R(\tau)]$
and
$J_C(\pi) := \mathbb{E}_{\tau \sim \pi}[C(\tau)]$
as its expected return and expected cumulative cost, respectively.

\begin{assumption}   \label{assump:cmdp-near-det}

\begin{enumerate}
    \item \textbf{Return-cost coverage.}
    There exists a constant $\alpha_F > 0$ such that
    \begin{equation}
    P_{\pi_{\beta}}\bigl( (R(\tau), C(\tau)) = F(s_1) | s_1 \bigr) \ge \alpha_F,
    \quad \forall s_1 \sim \mu.
    \label{eq:joint-coverage}
    \end{equation}

    \item \textbf{Near determinism.}\footnote{This assumption does not restrict the stochasticity of the initial state $s_1 \sim \mu$.}
    There exist deterministic functions $T : \mathcal{S} \times \mathcal{A} \to \mathcal{S}$, $r : \mathcal{S} \times \mathcal{A} \to \mathbb{R}$, and $c : \mathcal{S} \times \mathcal{A} \to [0, c_{\max}]$ such that, for all $(s,a)$ and all time steps $t$,
    \begin{equation}
    P\bigl(
      r_t \neq r(s,a)
      \;\text{or}\;
      c_t \neq c(s,a)
      \;\text{or}\;
      s_{t+1} \neq T(s,a)
      |
      s_t = s, a_t = a
    \bigr) \le \varepsilon.
    \label{eq:near-det-cmdp}
    \end{equation}

    \item \textbf{Consistency of $F$.}
    For every transition $(s,a,s')$ under the deterministic dynamics $T$, the conditioning function $F(s) = (F_R(s), F_C(s))$ satisfies
    \begin{equation}
    F_R(s) = F_R(s') + r(s,a),
    \qquad
    F_C(s) = F_C(s') + c(s,a).
    \label{eq:consistency-F}
    \end{equation}
\end{enumerate}
\end{assumption}

%%%%%%%%%%%%%

\begin{theorem}[Joint alignment with respect to the conditioning function in CMDPs]
\label{thm:cmdp-alignment}
Consider a finite-horizon CMDP, a behavior policy $\pi_\beta$, and a conditioning function $F : \mathcal{S} \to \mathbb{R}^2$. Suppose Assumption~\ref{assump:cmdp-near-det} is satisfied. Let $\pi^{\mathrm{CDT}}_F$ denote the CDT policy in~\eqref{eq:cdt-optimal-policy}. Then there exists a universal constant $C > 0$ such that
\begin{align}
\mathbb{E}_{s_1 \sim \mu}\bigl[F_R(s_1)\bigr] - J_R\bigl(\pi^{\mathrm{CDT}}_F\bigr)
&\;\le\;
C \,\varepsilon \Bigl(\frac{1}{\alpha_F} + 2\Bigr) H^2,
\label{eq:reward-gap-cmdp}
\\
\mathbb{E}_{s_1 \sim \mu}\bigl[F_C(s_1)\bigr] - J_C\bigl(\pi^{\mathrm{CDT}}_F\bigr)
&\;\le\;
C \,\varepsilon \Bigl(\frac{1}{\alpha_F} + 2\Bigr) H^2.
\label{eq:cost-gap-cmdp}
\end{align}
% Moreover, there exist CMDPs for which these bounds are tight up to constant factors.
\end{theorem}
\noindent The proof is provided in Appendix~\ref{sec:proof1}.

Theorem~\ref{thm:cmdp-alignment} formalizes the alignment behavior of return--cost conditioned sequence modeling in CMDPs. Under joint return--cost coverage and near-deterministic dynamics, the optimal solution to the CDT objective induced by a fixed conditioning function $F$ produces a policy whose expected return and expected cumulative cost track the prescribed targets, with a mismatch bounded by $O\bigl(\varepsilon(\tfrac{1}{\alpha_F}+2)H^2\bigr)$. This result provides a theoretical basis for viewing RTG/CTG conditioning as a mechanism for zero-shot control of the return--cost trade-off: when the target profile specified by $F$ is sufficiently represented in the behavior data (large $\alpha_F$), conditioning can be realized reliably up to a controllable error. At the same time, the explicit dependence on $\alpha_F$ clarifies a fundamental limitation of naive target conditioning in offline safe RL. In realistic datasets where safe high-return trajectories are scarce, $\alpha_F$ can be small for ambitious targets, and the induced policy may substantially deviate from the intended profile, leading to unstable return--cost trade-offs and even constraint violations. In particular, simply specifying a high RTG together with a low CTG does not reliably yield favorable return--cost trade-offs on challenging datasets with limited coverage of safe, high-performance behaviors.

This motivates augmenting the plain CDT backbone with data-aware mechanisms that reshape the effective training distribution and provide additional optimization signals beyond maximum likelihood. Specifically, we seek to bias learning toward trajectories with more desirable return--cost profiles while retaining the flexibility of target conditioning for multi-threshold, zero-shot deployment. Notably, the same proof technique also extends to multi-goal conditioning by viewing each conditioning coordinate as an additional scalar, reward-like signal that obeys a Markov-consistent recursion along transitions, leading to analogous coverage-dependent alignment bounds in higher-dimensional conditioning spaces. In the following, we instantiate this principle by introducing a regularized CDT variant that incorporates trajectory-level reweighting and value-based terms, aligning with a broader line of work showing that reweighting and regularization can substantially improve DT-style and CSM-based offline RL under distribution shift and data imbalance~\cite{brandfonbrener2022does,hu2024q,wang2024critic,bairebalancing}.

% This motivates augmenting the plain CDT backbone with data-aware mechanisms that reshape the effective training distribution and provide additional optimization signals beyond maximum likelihood. Specifically, we seek to bias learning toward trajectories with more desirable return--cost profiles while retaining the flexibility of target conditioning for multi-threshold, zero-shot deployment. In the following, we instantiate this principle by introducing a regularized CDT variant that incorporates trajectory-level reweighting and value-based terms, aligning with a broader line of work showing that reweighting and regularization can substantially improve DT-style and CSM-based offline RL under distribution shift and data imbalance~\cite{hu2024q,wang2024critic,bairebalancing}.

\subsection{\algw}\label{sec:weighted-q-guidance}

\begin{table}[t]
\centering
\small
\caption{Optimization objectives for DT-based CSM methods.}
\label{tab:objective-comparison}
\setlength{\tabcolsep}{6pt}
\renewcommand{\arraystretch}{1.3}
\resizebox{\textwidth}{!}{
\begin{tabular}{ll}
\toprule
\textbf{Method} & \textbf{Optimization Objective} \\
\midrule
Offline RL & \\
DT \cite{chen2021decision}
&
$
\mathcal{L}_{\text{DT}}(\theta) = - {\textstyle\sum_{\tau \in \mathcal{D}}} {\textstyle\sum_{1 \leq t \leq H}} \log \pi_{\theta}(a_t | s_t, \bar{R}_t, \bar{\tau}_{t-1}^{K} )
$
\\

QT \cite{hu2024q}
&
$
\mathcal{L}_{\text{DT}}(\theta)- \eta \mathbb{E}_{\tau \sim \mathcal{D}} \mathbb{E}_{s_t \sim \tau, a_t \sim \pi_{\theta}} [Q^{\pi_{\theta}}(s_t, a_t) ]
$
\\

RDT \cite{bairebalancing}
&
$
\mathcal{L}_{\text{DT}}(\theta) + \alpha \mathbb{E}_{\tau \sim \mathcal{D}_e} \sum_{i=t}^{H} \text{KL} \big[ \pi_{\theta} (\cdot | s_t, \bar{R}_t, \bar{\tau}_{t-1}^{K}) \Vert \pi^{e} (\cdot | s_t) \big]
$
\\
RVDT \cite{bairebalancing}
&
$ 
\mathcal{L}_{\text{DT}}(\theta) - \eta \mathbb{E}_{\tau \sim \mathcal{D}} \mathbb{E}_{s_t \sim \tau, a_t \sim \pi_{\theta}} [Q^{\pi_\theta}(s_t, a_t) ] 
+ \alpha \mathbb{E}_{\tau \sim \mathcal{D}_e} {\textstyle\sum_{t=1}^{H}} \text{KL} \big[ \pi_{\theta} (\cdot | s_t, \bar{R}_t, \bar{\tau}_{t-1}^{K}) \Vert \pi^{e} (\cdot | s_t) \big]
$
\\
\midrule
Offine Safe RL & \\
CDT \cite{liu2023constrained}
&
$
\mathcal{L}_{\text{CDT}}(\pi) = - {\textstyle\sum_{\tau \in \mathcal{D}}} {\textstyle\sum_{1 \leq t \leq H}} \log \pi(a_t | s_t, \bar{R}_t, \bar{C}_t,\bar{\tau}_{t-1}^{K})
$
\\
\algwb~(ours)
&
$
- {\sum_{\tau \in \mathcal{D}}}{\sum_{t=1}^{H}}W_{\tau} \log \pi_{\theta} \left(a_t | s_t, \bar{R}_t, \bar{C}_t, \bar{\tau}_{t-1}^{K}\right) - \eta \mathbb{E}_{\tau \sim \mathcal{D}} \mathbb{E}_{s_t \sim \tau, a_t \sim \pi_{\theta}} [Q^{\pi_\theta}(s_t, a_t) ]
$
\\
\algcb~(ours)
&
$
- {\sum_{\tau \in \mathcal{D}}}{\sum_{t=1}^{H}}W_{\tau} \log \pi_{\theta} \left(a_t | s_t, \bar{R}_t, \bar{C}_t, \bar{\tau}_{t-1}^{K}\right) - \eta \mathbb{E}_{\tau \sim \mathcal{D}} \mathbb{E}_{s_t \sim \tau, a_t \sim \pi_{\theta}} [Q^{\pi_\theta}(s_t, a_t) ] + \lambda \mathbb{E}_{s_t \sim \tau, a_t \sim \pi_{\theta}} [C^{\pi_\theta}(s_t, a_t) ]
$
\\
\bottomrule
\end{tabular}
}
\end{table}

Building on CDT, we introduce two data-aware training ingredients to improve algorithm performance while retaining the RTG/CTG conditioning interface for multi-threshold, zero-shot deployment.
Specifically, we (i) reshape the effective training distribution via trajectory-level reweighting toward more desirable return--cost profiles, and (ii) incorporate value-based guidance as an additional optimization signal to discourage overly conservative behaviors.
This design follows the technical progression summarized in Table~\ref{tab:objective-comparison}: it generalizes the sub-dataset resampling used in DT-style methods (e.g., RDT/RVDT) to a more flexible trajectory-level scheme, and adapts the Q-value regularization principle of QT to the constrained CDT setting.

Let $R(\tau)$ and $C(\tau)$ denote the cumulative return and cumulative cost of a trajectory $\tau \in \mathcal{D}$ respectively. We introduce a nonnegative weighting function
\begin{equation}
\mathcal{W} : \mathbb{R}^2 \to \mathbb{R}_{>0}, \quad W_{\tau} = \mathcal{W} \big(R(\tau), C(\tau)\big),
\end{equation}
which assigns larger weights to trajectories that exhibit more favorable return--cost trade-offs. It suffices to note that $W_{\tau}$ depends jointly on return and cost, and therefore encodes both performance and safety preferences at the trajectory level. Given these weights, consider the following weighted variant of the CDT loss:
\begin{equation}
\label{eq:wcdt-loss}
\mathcal{L}(\theta)=- {\sum_{\tau \in \mathcal{D}}}
{\sum_{t=1}^{H}}
W_{\tau}\,
\log \pi_{\theta}\left(a_t | s_t, \bar{R}_t, \bar{C}_t, \bar{\tau}_{t-1}^{K}\right).
\end{equation}
The loss \eqref{eq:wcdt-loss} preserves the behavior-cloning structure of CDT while explicitly rebalancing the empirical training distribution toward trajectories that are simultaneously high-return and low-cost.

To connect this formulation with existing DT-based methods, we draw a connection to the KL-based expert regularization used by RDT and RVDT in prior work \cite{bairebalancing}. Consider the RDT objective:
\begin{equation}
\label{eq:rdt-kl}
\mathcal{L}_{\text{RDT}}(\theta)
=
\mathbb{E}_{\tau \sim \mathcal{D}}
\Big[ \sum_{t=1}^{H}
-\log \pi_{\theta}\left(a_t | s_t, g(\tau_t), \bar{\tau}_{t-1}^{K}\right) \Big]
+
\alpha\,
\mathbb{E}_{\tau \sim \mathcal{D}_e}
\Big[\sum_{t=1}^{H}
\mathrm{KL}\left(
\pi_{\theta}(\cdot | s_t, g(\tau_t), \bar{\tau}_{t-1}^{K})
\,\Vert\,
\pi^{e}(\cdot | s_t)
\right)\Big],
\end{equation}
where $\mathcal{D}_e \subset \mathcal{D}$ is a subset of near-optimal trajectories and $\pi^{e}$ is an expert policy obtained from imitation learning on $\mathcal{D}_e$. The following result shows that the expert KL term implicitly induces a particular trajectory-level weighting scheme, and can therefore be seen as a special case of our general reweighting formulation.

\begin{proposition}[KL regularization as a special case of trajectory weighting]
\label{prop:kl-as-weighting}
Assume that $\pi_{\theta}$ is parameterized as a factorized Gaussian policy with a fixed standard deviation and that $\pi^{e}$ is fitted on $\mathcal{D}_e$ via maximum-likelihood imitation learning. Then the RDT objective in Eq.~\eqref{eq:rdt-kl} can be rewritten as a weighted DT loss:
\begin{equation}
\label{eq:weighted-dt-special}
\arg\min_{\pi_\theta} \mathcal{L}_{\text{RDT}}(\theta)
=
\arg\min_{\pi_\theta}
\mathbb{E}_{\tau \sim \mathcal{D}}
\Big[
W^{\text{RDT}}_{\tau}
\sum_{t=1}^{H}
-\log \pi_{\theta}\left(a_t | s_t, g(\tau_t), \bar{\tau}_{t-1}^{K}\right)
\Big],
\end{equation}
with trajectory weights
$W^{\text{RDT}}_{\tau}=1 + \alpha\, \mathbb{I}\left[\tau \in \mathcal{D}_e\right], \ \alpha \geq 0.$
\end{proposition}
\noindent The proof is provided in Appendix~\ref{sec:proof2}.

Proposition~\ref{prop:kl-as-weighting} indicates that KL-based expert regularization effectively performs a weighted resampling of expert trajectories, where all trajectories in $\mathcal{D}_e$ share an increased weight and all others retain unit weight. In this sense, RDT is a special case of trajectory-level reweighting that typically constructed based on return ranking. Our formulation in Eq.~\eqref{eq:wcdt-loss} generalizes this mechanism by allowing $W_{\tau}$ to depend on both return and cost, and by instantiating it on the constrained CDT backbone. 
% Combined with the value guidance term in Eq.~\eqref{eq:algwb-loss}, \algwb~extends KL-based expert regularization by allowing return--cost dependent weights and applying it to the constrained CDT backbone.

In addition, prior results from QT \cite{hu2024q} and RVDT \cite{bairebalancing} indicate that value-based regularization can improve the robustness of DT-style policies on challenging offline datasets. Accordingly, we augment Eq.~\eqref{eq:wcdt-loss} with a value guidance term and obtain
\begin{equation}
\label{eq:algwb-loss}
\mathcal{L}_{\text{\algwb}}(\theta)
=
- {\sum_{\tau \in \mathcal{D}}}
{\sum_{t=1}^{H}}
W_{\tau}\,
\log \pi_{\theta}\left(a_t | s_t, \bar{R}_t, \bar{C}_t, \bar{\tau}_{t-1}^{K}\right)
-
\eta\,
\mathbb{E}_{\tau \sim \mathcal{D}}
\mathbb{E}_{s_t \sim \tau,\, a_t \sim \pi_{\theta}}
\left[ Q^{\pi_\theta}(s_t, a_t) \right],
\end{equation}
where $\eta > 0$ is a regularization coefficient and $Q^{\pi_\theta}$ denotes the action-value function of the current CSM policy $\pi_\theta$. We term the resulting algorithm as \algw~(\algwb). The first term in Eq.~\eqref{eq:algwb-loss} performs explicit trajectory-level return-cost reweighting on CDT, and the second term provides an optimization signal that encourages actions with higher long-term value under $\pi_\theta$. 

\subsection{\algc}
\label{sec:main-algorithm}

While \algwb already biases learning toward trajectories that are both high-return and low-cost, it lacks an explicit cost-penalization mechanism to further discourage unsafe behavior. A standard approach in safe reinforcement learning is to formulate policy learning as a constrained optimization problem in CMDPs and to introduce Lagrangian penalties to balance performance and safety. Following this perspective, we derive a cost-penalized training objective together with an adaptive update rule for its coefficient, while avoiding coupling CDT to a single fixed cost threshold.
To motivate the penalty term, consider minimizing the weighted CDT loss with value guidance subject to an upper bound $\kappa$ on the expected cumulative cost of the learned policy. A CMDP-style formulation can be written as:
\begin{equation}
\label{eq:algc-cmdp}
\begin{aligned}
\min_{\theta}\;&
- {\textstyle\sum_{\tau \in \mathcal{D}}}
{\textstyle\sum_{t=1}^{H}}
W_{\tau}\,
\log \pi_{\theta}\left(a_t \mid s_t, \bar{R}_t, \bar{C}_t, \bar{\tau}_{t-1}^{K}\right)
-
\eta\,
\mathbb{E}_{\tau \sim \mathcal{D}}
\mathbb{E}_{s_t \sim \tau,\, a_t \sim \pi_{\theta}}
\left[ Q^{\pi_\theta}(s_t, a_t) \right] \\
\text{s.t.}\;&
J_c(\theta) :=
\mathbb{E}_{\tau \sim \mathcal{D}}
\mathbb{E}_{s_t \sim \tau,\, a_t \sim \pi_{\theta}}
\left[ C^{\pi_\theta}(s_t, a_t) \right]
\leq \kappa,
\end{aligned}
\end{equation}
where $\kappa > 0$ denotes a target cost level. Introducing a Lagrange multiplier $\lambda \geq 0$ leads to the Lagrangian
\begin{equation}
\label{eq:algc-lagrangian}
\begin{aligned}
\mathcal{L}_{\text{lag}}(\theta, \lambda)
&=
- {\textstyle\sum_{\tau \in \mathcal{D}}}
{\textstyle\sum_{t=1}^{H}}
W_{\tau}\,
\log \pi_{\theta}\left(a_t \mid s_t, \bar{R}_t, \bar{C}_t, \bar{\tau}_{t-1}^{K}\right)
-
\eta\,
\mathbb{E}_{\tau \sim \mathcal{D}}
\mathbb{E}_{s_t \sim \tau,\, a_t \sim \pi_{\theta}}
\left[ Q^{\pi_\theta}(s_t, a_t) \right] \\
&\quad +
\lambda\Big(
\mathbb{E}_{\tau \sim \mathcal{D}}
\mathbb{E}_{s_t \sim \tau,\, a_t \sim \pi_{\theta}}
\left[ C^{\pi_\theta}(s_t, a_t) \right]
- \kappa
\Big).
\end{aligned}
\end{equation}
Eq.~\eqref{eq:algc-lagrangian} serves as a motivating formulation rather than as a strict optimisation problem: it suggests that a natural way to bias the policy towards safer behaviours is to augment the training loss with the expected cost term, weighted by an appropriate coefficient. 
Imposing this constraint as a hard requirement tied to a fixed $\kappa$ would also be at odds with the zero-shot adaptation capability of conditional sequence modeling, where a single CDT policy is expected to generalise across a range of cost thresholds. 
Instead, we treat the cost term as a flexible regulariser whose influence is controlled by an adaptive coefficient $\lambda$.

Finally, we arrive at our offline safe RL algorithm, referred to as \algc~(\algcb), by optimizing the following cost-penalized actor objective:
\begin{equation}
\label{eq:algc-loss}
\begin{aligned}
\mathcal{L}_{\text{\algcb}}(\theta)
&=
- {\textstyle\sum_{\tau \in \mathcal{D}}}
{\textstyle\sum_{t=1}^{H}}
W_{\tau}\,
\log \pi_{\theta}\left(a_t \mid s_t, \bar{R}_t, \bar{C}_t, \bar{\tau}_{t-1}^{K}\right)
-
\eta\,
\mathbb{E}_{\tau \sim \mathcal{D}}
\mathbb{E}_{s_t \sim \tau,\, a_t \sim \pi_{\theta}}
\left[ Q^{\pi_\theta}(s_t, a_t) \right] \\
&\quad +
\lambda\,
\mathbb{E}_{\tau \sim \mathcal{D}}
\mathbb{E}_{s_t \sim \tau,\, a_t \sim \pi_{\theta}}
\left[ C^{\pi_\theta}(s_t, a_t) \right],
\end{aligned}
\end{equation}
where we reinterpret $\lambda > 0$ as an adaptive regularization coefficient rather than an exact Lagrange multiplier associated with a fixed hard constraint, thereby using it to balance performance and safety.
To obtain such an adaptive update rule, we note that the derivative of the Lagrangian in Eq.~\eqref{eq:algc-lagrangian} with respect to $\lambda$ is proportional to $J_c(\theta) - \kappa$. This suggests adjusting $\lambda$ according to whether the current policy is above or below the target cost level. In practice, let $\hat{J}_c(\theta)$ denote a minibatch estimate of
\[
J_c(\theta) =
\mathbb{E}_{\tau \sim \mathcal{D}}
\mathbb{E}_{s_t \sim \tau,\, a_t \sim \pi_{\theta}}
\left[ C^{\pi_\theta}(s_t, a_t) \right].
\]
We update $\lambda$ using
\begin{equation}
\label{eq:lambda-update}
\lambda \leftarrow
\big[\lambda + \beta\big(\hat{J}_c(\theta) - \kappa\big)\big]_+,
\end{equation}
where $\beta > 0$ is a step size and $[\cdot]_+$ denotes projection onto the nonnegative reals. When the estimated expected cost $\hat{J}_c(\theta)$ exceeds $\kappa$, $\lambda$ is increased and the cost penalty in Eq.~\eqref{eq:algc-loss} becomes stronger; when the $\hat{J}_c(\theta)$ falls below $\kappa$, the penalty is gradually relaxed. This update follows the primal--dual intuition of adjusting a penalty coefficient based on the current level of constraint violation, and provides an automatic way to tune the strength of safety regularisation during training. Importantly, since the CDT architecture explicitly conditions on cost tokens, the resulting policy remains compatible with evaluation under multiple cost thresholds without retraining, and $\kappa$ should be interpreted as a training-time reference level, rather than a hard constraint that fixes a single deployment threshold.

In practice, the trajectory weights $W_{\tau}$ in Eq.~\eqref{eq:algc-loss} are instantiated by a return-cost shaped function that increases with return and decreases with cost relative to a prescribed limit. Specifically, we use
\begin{equation}
\label{eq:w-def}
W_\tau
=
\exp\big(\alpha\,R(\tau)\big)\,
\operatorname{sigmoid}\big(\gamma(c_{\text{lim}} - C(\tau))\big),
\end{equation}
where $\alpha$ and $\gamma$ control the sensitivity of $W_\tau$ to return and cost, and $c_{\text{lim}}$ denotes a trajectory-level cost reference. This choice is consistent with the general form $W_{\tau} = \mathcal{W}\big(R(\tau), C(\tau)\big)$ introduced in Section~\ref{sec:weighted-q-guidance}, and makes explicit how return and cost jointly shape the effective training distribution.

Figure~\ref{fig:weight} illustrates the effect of Eq.~\eqref{eq:w-def} on representative tasks. Each point corresponds to a trajectory in the offline dataset, with cumulative cost on the horizontal axis and return on the vertical axis; colour intensity indicates the magnitude of $W_{\tau}$ (darker means larger weight). As costs increase beyond the eference $c_{\text{lim}}$, trajectories are progressively down-weighted, whereas trajectories that achieve higher returns while remaining within (or near) the cost reference are emphasised. This visualisation highlights how trajectory-level return--cost reweighting reshapes the training signal toward safer and more performant regions of the offline data.

Putting these components together, \algcb~can be viewed as a CMDP-inspired extension of CDT within the CSM paradigm.
It retains the target-conditioning interface of CDT for multi-threshold, zero-shot evaluation, while augmenting maximum-likelihood training with three complementary ingredients: trajectory-level return--cost reweighting to emphasise favourable behaviours in the offline data, value-based guidance to refine the policy beyond pure imitation, and an explicit cost penalty whose coefficient is adaptively tuned to control safety in expectation.
In this way, \algcb~aligns with the DT-based objective design path summarised in Table~\ref{tab:objective-comparison}, while remaining directly applicable to zero-shot adaptation across multiple cost thresholds at evaluation time.

\begin{figure}[h]
    \centering
    \begin{subfigure}[t]{0.30\textwidth}
        \centering
        \includegraphics[width=\linewidth]{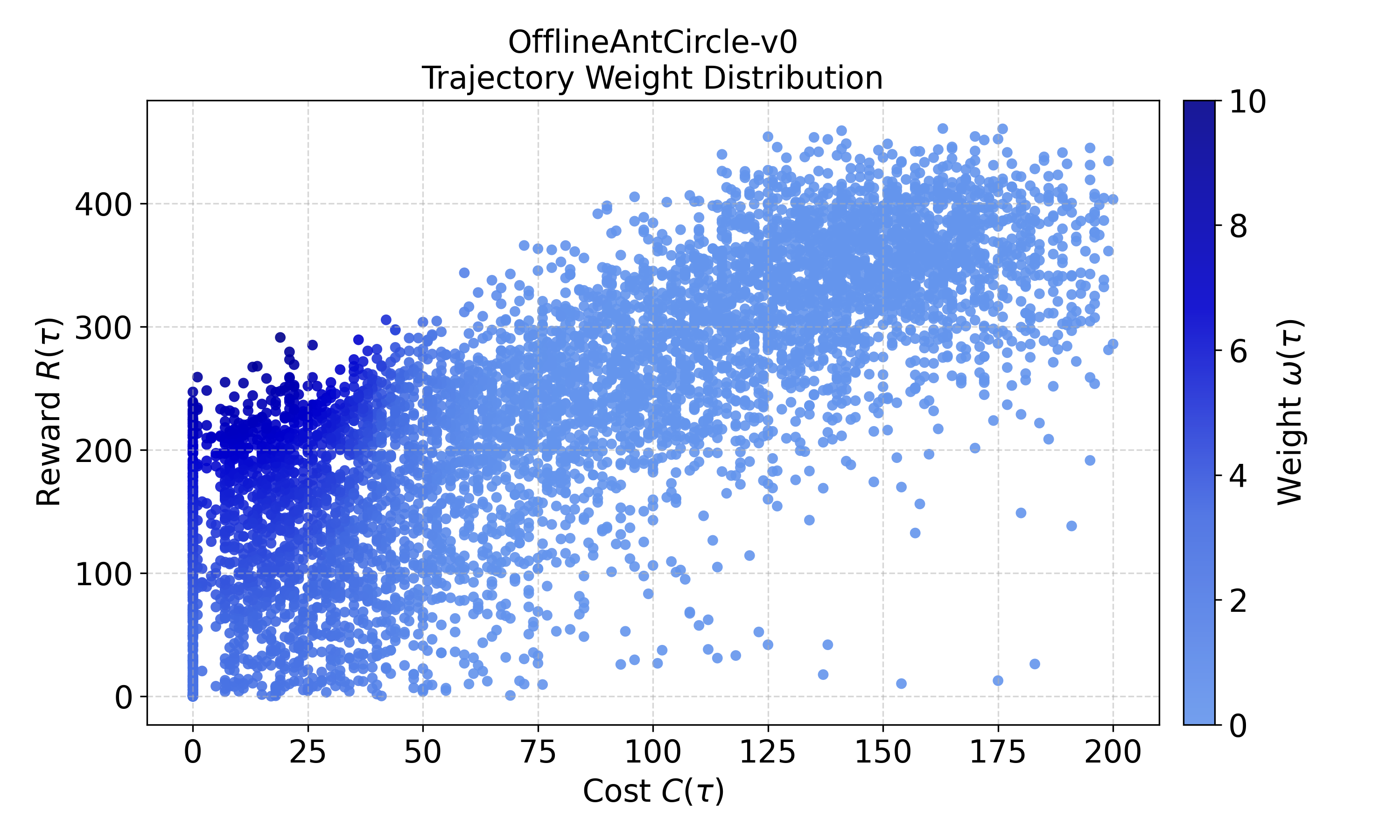}
        \caption{AntCircle}
    \end{subfigure}
    \hfill
    \begin{subfigure}[t]{0.30\textwidth}
        \centering
        \includegraphics[width=\linewidth]{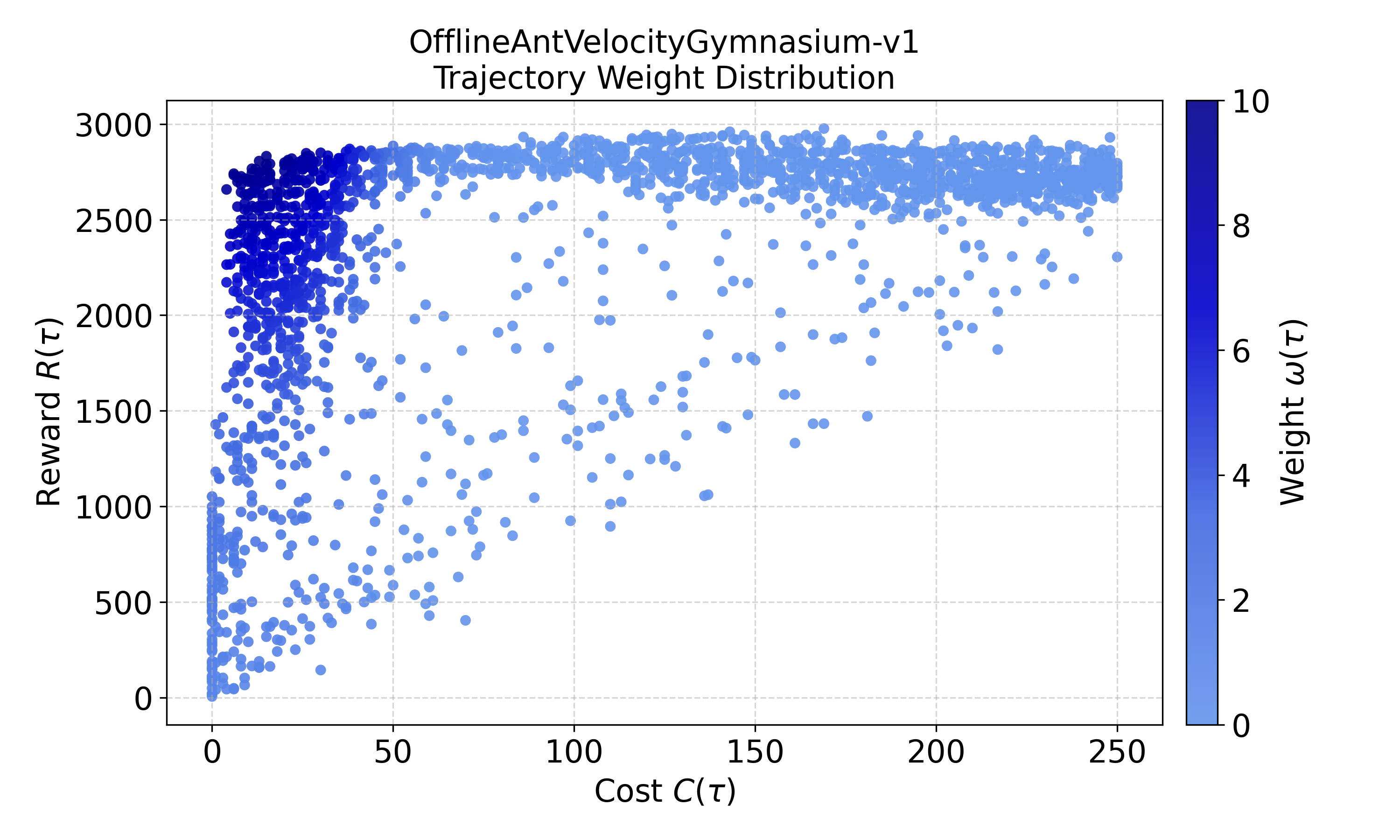}
        \caption{AntVelocity}
    \end{subfigure}
    \hfill
    \begin{subfigure}[t]{0.30\textwidth}
        \centering
        \includegraphics[width=\linewidth]{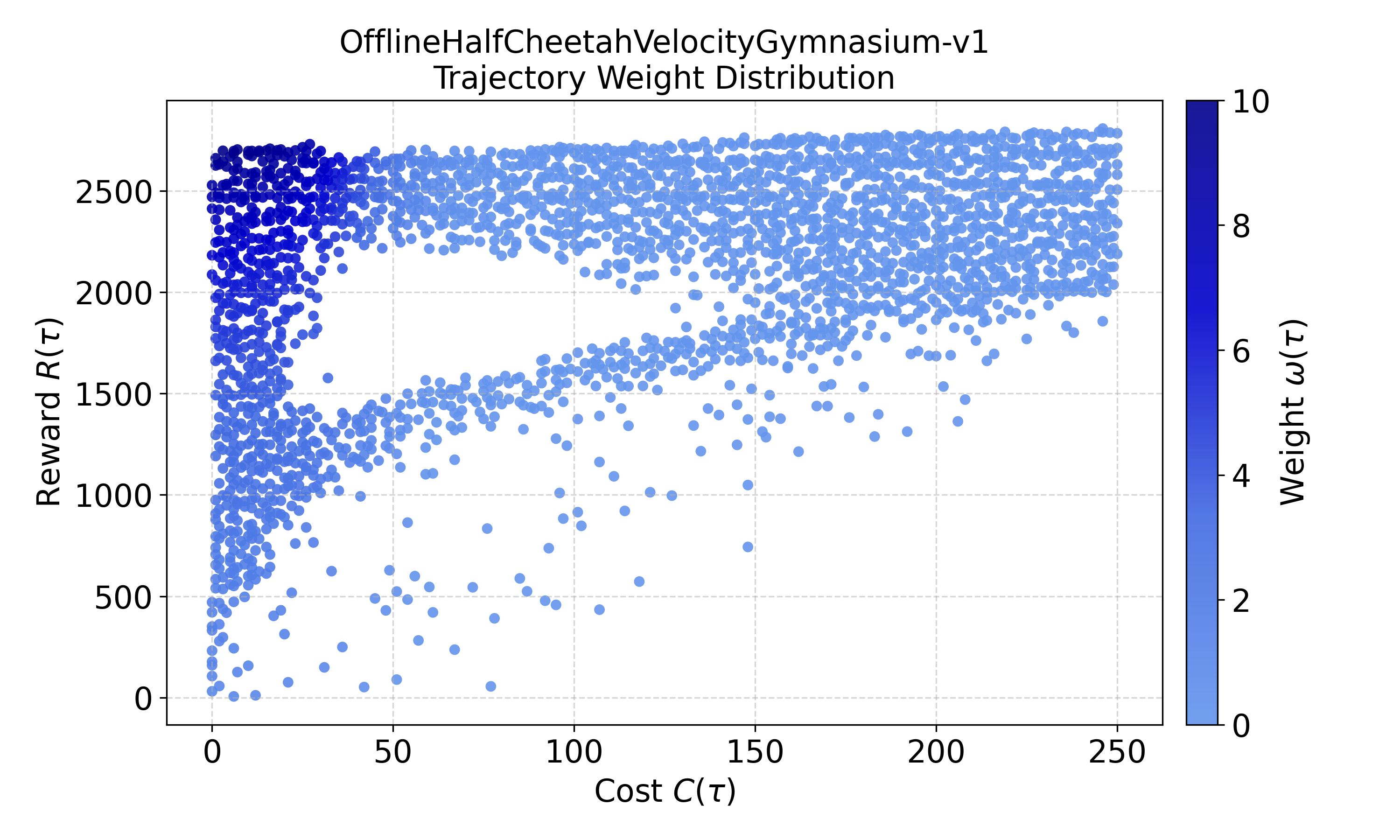}
        \caption{CheetahVelocity}
    \end{subfigure}
    \caption{
        Return--cost distributions of offline trajectories on representative tasks.
        Each point denotes a trajectory, with cumulative cost on the $x$-axis and return on the $y$-axis; colour intensity indicates the trajectory weight $W_{\tau}$ (darker means larger weight) defined in Eq.~\eqref{eq:w-def}.
    }
    \label{fig:weight}
\end{figure}

\subsection{Training and inference}
\label{sec:training-inference}

We instantiate \algcb~using a transformer-based architecture.
The policy $\pi_\theta$ is implemented as an autoregressive transformer that takes as input a context window of past states, actions, and target return-cost tokens, and outputs the parameters of a stochastic policy.
Following common practice in sequence-modeling RL, $\pi_\theta$ is modelled as a factorised Gaussian distribution whose mean and log-variance are predicted by separate fully connected output heads at each step of the sequence.
For a CDT policy $\pi_{\theta}\left(\cdot \mid s_i, \bar{R}_i, \bar{C}_i, \bar{\tau}_{i-1}^{K}\right)$ at timestep $i$, $\bar{R}_i$ and $\bar{C}_i$ denote the RTG and CTG tokens obtained by cumulatively summing future rewards and costs along the underlying trajectory, and $\bar{\tau}_{i-1}^{K}$ denotes the length-$K$ history of past states and actions. When $i-K < 1$, the context window $\bar{\tau}_{i-1}^{K}$ is truncated to the available prefix of the trajectory.

\begin{algorithm}[t]
    \caption{\algcb: \algc}\label{alg:RCDT}
    \KwIn{Sequence length $K$, dataset $\mathcal{D}$, hyperparameters $\alpha$, $\gamma$, $\eta$, $\beta$, cost reference $c_{\text{lim}}$, target cost $\kappa$}
    \KwInit{Policy $\pi_\theta$, reward critic $Q_{\psi}$, cost critic $C_{\xi}$, coefficient $\lambda \ge 0$}
    Pre-compute cumulative return $R(\tau)$ and cumulative cost $C(\tau)$ for each trajectory $\tau \in \mathcal{D}$\;
    \For{$t = 1$ \KwTo $T$}{
        Sample a trajectory $\tau$ from $\mathcal{D}$\;
        Compute trajectory weight $W_{\tau}$ via Eq.~\eqref{eq:w-def}\;

        \tcp{Critic update}
        Sample a length-$K$ subsequence $\{(s_i, a_i, r_i, c_i)\}_{i=1}^K$ from $\tau$\;
        Form CDT inputs $\{(s_i, \bar{R}_i, \bar{C}_i, \bar{\tau}_{i-1}^{K})\}_{i=1}^K$ from the subsequence\;
        Sample next action $\hat{a}_{K} \sim \pi_{\theta}(\cdot \mid s_K, \bar{R}_K, \bar{C}_K, \bar{\tau}^{K}_{K-1})$\;
        Update $Q_{\psi}$ via offline TD learning using $(s_{K-1}, a_{K-1}, r_{K-1}, s_{K}, \hat{a}_{K})$\;
        Update $C_{\xi}$ via offline TD learning using $(s_{K-1}, a_{K-1}, c_{K-1}, s_{K}, \hat{a}_{K})$\;

        \tcp{Policy update}
        \For{$i = 1$ \KwTo $K$}{
            Sample action $\hat{a}_i \sim \pi_{\theta}(\cdot \mid s_i, \bar{R}_i, \bar{C}_i, \bar{\tau}_{i-1}^{K})$\;
            Evaluate $Q_{\psi}(s_i, \hat{a}_i)$ and $C_{\xi}(s_i, \hat{a}_i)$\;
        }
        Update policy $\pi_\theta$ by minimising $\mathcal{L}_{\text{\algcb}}(\theta)$ in Eq.~\eqref{eq:algc-loss}
        with $W_{\tau}$, $Q_{\psi}$ and $C_{\xi}$\;

        \tcp{Adaptive update of the cost coefficient}
        Estimate $\hat{J}_c(\theta)$ as the average $C_{\xi}(s_i, \hat{a}_i)$ over the current subsequence\;
        Update $\lambda \leftarrow \big[\lambda + \beta\big(\hat{J}_c(\theta) - \kappa\big)\big]_+$\;
    }
    \Return $\pi_\theta$\;
\end{algorithm}

\textbf{Training.}
The overall training pipeline of \algcb~is summarised in Algorithm~\ref{alg:RCDT}.
Prior to training, the cumulative return $R(\tau)$ and cumulative cost $C(\tau)$ are computed for each trajectory $\tau \in \mathcal{D}$.
These statistics are used to instantiate the trajectory weights $W_\tau$ via Eq.~\eqref{eq:w-def}, which increase with return and decrease with cost relative to a reference limit $c_{\text{lim}}$.
During each training iteration, a trajectory $\tau$ is sampled from the dataset, a length-$K$ subsequence is extracted, and the corresponding CDT inputs is $(s_i, \bar{R}_i, \bar{C}_i, \bar{\tau}_{i-1}^{K})$.

Value learning is carried out using standard offline temporal-difference (TD) methods.
The reward critic $Q_{\psi}$ approximates the action-value function $Q^{\pi_\theta}$, while the cost critic $C_{\xi}$ approximates the cost-to-go function $C^{\pi_\theta}$.
Given a transition $(s_{K-1}, a_{K-1}, r_{K-1}, c_{K-1}, s_K)$ from the sampled subsequence, a next action $\hat{a}_K$ is drawn from the current policy
$\pi_{\theta}(\cdot \mid s_K, \bar{R}_K, \bar{C}_K, \bar{\tau}_{K-1}^{K})$, and TD targets are constructed using $(s_{K-1}, a_{K-1}, r_{K-1}, s_K, \hat{a}_K)$ and $(s_{K-1}, a_{K-1}, c_{K-1}, s_K, \hat{a}_K)$ for the reward and cost critics, respectively.
In practice, standard stabilisation techniques such as target networks~\cite{lillicrap2015continuous} and double Q-learning~\cite{fujimoto2018addressing} are employed for both $Q_{\psi}$ and $C_{\xi}$.

The policy parameters $\theta$ are updated by minimising the actor loss in Eq.~\eqref{eq:algc-loss}.
Concretely, for each CDT input in the subsequence, an action $\hat{a}_i$ is sampled from
$\pi_{\theta}(\cdot \mid s_i, \bar{R}_i, \bar{C}_i, \bar{\tau}_{i-1}^{K})$, and the corresponding critic evaluations $Q_{\psi}(s_i, \hat{a}_i)$ and $C_{\xi}(s_i, \hat{a}_i)$ are used to form a Monte Carlo estimate of the expectations in Eq.~\eqref{eq:algc-loss}.
The trajectory weight $W_\tau$ is applied to the negative log-likelihood term to realise the return-cost reweighting effect.
The critic-based terms provide value and cost guidance, and the combined stochastic gradient is used to update $\theta$ via gradient descent.

The coefficient $\lambda$ that controls the strength of the cost penalty is updated adaptively according to Eq.~\eqref{eq:lambda-update}.
At each iteration, an empirical estimate $\hat{J}_c(\theta)$ of
$
J_c(\theta) =
\mathbb{E}_{\tau \sim \mathcal{D}}
\mathbb{E}_{s_t \sim \tau,\, a_t \sim \pi_{\theta}}
\left[ C^{\pi_\theta}(s_t, a_t) \right]
$
is obtained by averaging $C_{\xi}(s_i, \hat{a}_i)$ over the sampled subsequence.
The update
$\lambda \leftarrow [\lambda + \beta(\hat{J}_c(\theta) - \kappa)]_+$ 
then increases $\lambda$ when the current policy exceeds the reference cost level $\kappa$, and decreases it otherwise.
As discussed in Section~\ref{sec:main-algorithm}, $\lambda$ is therefore interpreted as an adaptive regularisation coefficient rather than an exact Lagrange multiplier, and its update rule is inspired by Lagrangian dual ascent in safe RL.

\textbf{Inference.}
The inference procedure of \algcb~closely follows that of CDT and other CSM methods, with the additional capability of conditioning on cost.
At each decision step, the model receives four input sequences: the target RTG, the target CTG, the past states, and the past actions.
Given these inputs, the transformer autoregressively predicts a sequence of $K$ actions, of which only the final action in the window is executed in the environment.

After executing the selected action and observing the realised reward and cost, both the RTG and CTG tokens are updated for the next timestep.Specifically, the target RTG is decreased by the observed reward, and the target CTG is decreased by the observed cost, mirroring the cumulative structure used to construct $\bar{R}_t$ and $\bar{C}_t$ in the offline dataset. These updated tokens are appended to the conditioning sequence and fed back into the transformer to generate the next action sequence.
Different evaluation thresholds on the cumulative cost are imposed by choosing different initial CTG values without retraining the model.
In this way, \algcb~preserves the zero-shot adaptation capability of CDT across a range of safety requirements, while leveraging trajectory-level reweighting, value guidance, and cost regularisation to produce policies that are both high-performing and safety-aware.

\section{Experiments} \label{sec:exp}
We conduct a comprehensive empirical evaluation of \algcb~on the DSRL benchmark~\cite{liu2023datasets}. 
In Section~\ref{sec:exp1}, we compare \algcb~against a broad range of representative offline safe RL baselines spanning multiple methodological paradigms, including imitation learning~\cite{koiralalatent}, value-based methods~\cite{fujimoto2019off,kumar2019stabilizing,xu2022constraints}, CSM-based approaches~\cite{liu2023constrained}, distribution-matching methods~\cite{lee2022coptidice}, as well as data-quality and feasible-region methods~\cite{gong2025offline,zhengsafe}. 
In Section~\ref{sec:exp3}, ablation studies are conducted to analyze how each key component of \algcb~affects the safety--performance trade-off, focusing on the Lagrangian-style cost penalty, trajectory-level reweighting, and Q-value regularization. Besides, To further assess robustness of \algcb~ under more demanding evaluation settings, dataset variants from the original DSRL benchmark with (\romannumeral1) a lower proportion of expert-level trajectories and (\romannumeral2) more imbalanced data distributions are considered.

\textbf{Tasks.}
We evaluate \algcb~on the DSRL benchmark~\cite{liu2023datasets}, a public benchmark that includes a variety of continuous-control and robot locomotion tasks commonly used in prior offline safe RL research~\cite{guo2025constraint,lin2023safe,xu2022constraints,yao2024oasis,zhengsafe,gong2025offline,koiralalatent}. 
DSRL provides offline datasets collected from three widely used safety-critical simulators, including Mujoco-based SafetyGym~\cite{ji2023safety}, PyBullet-based BulletSafetyGym~\cite{gronauer2022bullet}, and the autonomous-driving simulator MetaDrive~\cite{li2022metadrive}.
In \emph{SafetyGym}, we consider \texttt{goal-reaching}, \texttt{pushing}, and \texttt{circle-tracking} tasks with two agent morphologies (\texttt{Point} and \texttt{Car}). 
The agent is rewarded for completing the intended objective, such as reaching a target location, pushing an object to a goal, or tracking a circular route, while safety cost is incurred when the agent enters hazardous regions and, for circle-style tasks, when it violates the safety boundary constraints. 
We additionally include velocity-constrained locomotion tasks with standard Mujoco agents, including \texttt{Swimmer}, \texttt{Hopper}, \texttt{HalfChee}, \texttt{Walker2d}, and \texttt{Ant}. 
In these tasks, reward encourages forward progress, whereas cost is triggered when the agent's speed exceeds a predefined limit.
\emph{BulletSafetyGym} provides analogous safety-critical locomotion tasks with different robot morphologies. 
We evaluate on \texttt{Run} and \texttt{Circle} tasks with multiple robots (\texttt{Ball}, \texttt{Car}, \texttt{Drone}, and \texttt{Ant}). 
In \texttt{Run}, the agent is rewarded for fast forward motion within boundaries, but incurs cost when leaving the corridor or exceeding a velocity cap. 
In \texttt{Circle}, the agent is rewarded for clockwise circular motion while being penalized for leaving a designated safe region.
Finally, in \emph{MetaDrive} we evaluate offline safe driving datasets with increasing scenario difficulty and traffic density (e.g., \texttt{Easy/Medium/Hard} combined with \texttt{Sparse/Mean/Dense}), which capture realistic safety--performance tensions in interactive driving. 
Here reward encourages route progress and stable driving, while safety cost is issued for events such as collisions and driving out of the road.
Across all suites, reward promotes task completion, whereas cost penalizes safety violations. 
This inherent reward--cost conflict makes it necessary to balance return maximization against cumulative cost constraints: aggressively pursuing high reward can lead to constraint violations, while overly conservative policies tend to satisfy constraints at the expense of performance.

\textbf{Baselines.} We compare \algcb~against the following baselines.
(\romannumeral1) Imitation learning: Behavior Cloning (BC) and its safety-aware variant BC-Safe~\cite{koiralalatent}.
(\romannumeral2) Value-based with constraint handling: CPQ~\cite{xu2022constraints}, BCQ-Lag, and BEAR-Lag, where the latter two are safety-constrained variants of BCQ~\cite{fujimoto2019off} and BEAR~\cite{kumar2019stabilizing} obtained by incorporating Lagrangian penalties.
(\romannumeral3) CSM-based: Constrained Decision Transformer (CDT)~\cite{liu2023constrained}, a representative CSM method that learns conditioned policies.
(\romannumeral4) Distribution matching: COptiDICE~\cite{lee2022coptidice}, which performs constrained policy optimization via occupancy-measure matching in the DICE family~\cite{lee2021optidice}.
(\romannumeral5) Data-quality and feasible-region methods: TraC~\cite{gong2025offline}, and FISOR~\cite{zhengsafe}, which explicitly reason about safe regions or trajectory-level feasibility.

\textbf{Implementation.}
Our implementation of \algcb~is built upon the CDT implementation provided in the OSRL\footnote{https://github.com/liuzuxin/OSRL/tree/main} project repository released by the CDT authors. 
We reproduce CDT using this codebase and its recommended hyperparameters. 
For the remaining baselines, we either run their publicly available implementations or report results from the original papers or representative prior work~\cite{koiralalatent,suboundary,gong2025offline,guo2025constraint}, depending on availability.
Additional implementation details are provided in Appendix~\ref{sec:imple-details}.

\textbf{Evaluation.}
We follow the standard \emph{constraint-variation evaluation} protocol for offline safe RL~\cite{liu2023datasets,koiralalatent,gong2025offline}, which assesses an algorithm's robustness and versatility under different levels of safety constraints. 
Concretely, each method is evaluated on every dataset under three target cumulative-cost thresholds, i.e., $\{20,40,80\}$ for SafetyGym, $\{10,20,40\}$ for BulletSafetyGym, and $\{10,20,40\}$ for MetaDrive, and the final score is obtained by averaging over these thresholds.
For methods that can be evaluated under multiple thresholds without retraining (BC, BC-Safe, CDT, and our \algcb), we train a single policy per dataset and evaluate its behavior across all thresholds in a zero-shot manner. 
Notably, although CDT and \algcb~involve a target cost level $\kappa$ during training, which is set as a fixed hyperparameter for each dataset, the learned policies remain CTG-conditioned and are evaluated across all target thresholds without additional retraining. 
By contrast, for threshold-dependent baselines whose training objective explicitly depends on the target constraint level, including BCQ-Lag, BEAR-Lag, COptiDICE, CPQ, FISOR, and TraC, we train separate models for different thresholds following their original training pipelines and the benchmark protocol to align with their intended training settings.
Overall, this protocol provides a comprehensive assessment of the return--cost trade-off across a range of cost requirements, and it also directly validates the zero-shot deployment capability of \algcb~under varying cost thresholds.

\textbf{Metrics.}
We report the \emph{normalized return} and \emph{cumulative cost} as the comparison metrics, 
consistent with prior offline safe RL literature~\cite{liu2023datasets,koiralalatent,gong2025offline}. Let $R_{\max}(\mathcal{T})$ and $R_{\min}(\mathcal{T})$ denote the maximum and minimum returns observed in dataset $\mathcal{T}$, respectively. 
The normalized return is computed by:
$$
R_{\text{normalized}}
= \frac{R_{\pi} - R_{\min}(\mathcal{T})}{R_{\max}(\mathcal{T}) - R_{\min}(\mathcal{T})},
$$
where $R_{\pi}$ denotes the evaluated return of policy $\pi$. 
The normalized cumulative cost is defined by the ratio between the evaluated cumulative cost $C_{\pi}$ and the target threshold $\zeta$, i.e. $C_{\text{normalized}}=C_\pi\big / \zeta $ with $\zeta > 0$, noting that all costs in our setting are non-negative. 

\begin{table}[t]
\centering
\caption{
Performance comparison across DSRL tasks. For each task, \algcb~reports the averaged normalized scores over three cost thresholds and 5 random seeds.
\textcolor{gray}{Gray}: Unsafe agents.  \textbf{Bold}: Safe agents with cost $< 1$.  \textcolor{blue}{Blue}: Safe agents with the highest reward.
$\uparrow$ means the higher the better. $\downarrow$ means the lower the better. }
\label{tab:main-result}
\resizebox{\textwidth}{!}{
\begin{tabular}{l|cc|cc|cc|cc|cc|cc|cc|cc|cc|cc}
\toprule
\multirow{2}{*}{Task} &
\multicolumn{2}{c}{\textbf{BC-ALL}} &
\multicolumn{2}{c}{\textbf{BCQ-Lag}} &
\multicolumn{2}{c}{\textbf{BEAR-Lag}} &
\multicolumn{2}{c}{\textbf{COptiDICE}} &
\multicolumn{2}{c}{\textbf{BC-Safe}} &
\multicolumn{2}{c}{\textbf{CPQ}} &
\multicolumn{2}{c}{\textbf{FISOR}} &
\multicolumn{2}{c}{\textbf{CDT}} &
\multicolumn{2}{c}{\textbf{TraC}} &
\multicolumn{2}{c}{\textbf{\algcb}} \\
\cmidrule(lr){2-3}
\cmidrule(lr){4-5}
\cmidrule(lr){6-7}
\cmidrule(lr){8-9}
\cmidrule(lr){10-11}
\cmidrule(lr){12-13}
\cmidrule(lr){14-15}
\cmidrule(lr){16-17}
\cmidrule(lr){18-19}
\cmidrule(lr){20-21}
& rew $\uparrow$ & cost $\downarrow$
& rew $\uparrow$ & cost $\downarrow$
& rew $\uparrow$ & cost $\downarrow$
& rew $\uparrow$ & cost $\downarrow$
& rew $\uparrow$ & cost $\downarrow$
& rew $\uparrow$ & cost $\downarrow$
& rew $\uparrow$ & cost $\downarrow$
& rew $\uparrow$ & cost $\downarrow$
& rew $\uparrow$ & cost $\downarrow$
& rew $\uparrow$ & cost $\downarrow$ \\
\midrule
\multicolumn{21}{l}{\textbf{Safety Gym:}} \\
PointCircle1 & 0.79 & 3.98 & 0.54 & 2.38 & 0.73 & 3.28 & 0.86 & 5.51 & \textbf{0.41} & \textbf{0.16} & \textbf{0.43} & \textbf{0.75} & 0.43 & 14.93 & \textbf{0.57} & \textbf{0.67} & \textbf{0.5} & \textbf{0.07} & \textcolor{blue}{\textbf{0.58}} & \textcolor{blue}{\textbf{0.56}}  \\
PointCircle2 & 0.66 & 4.17 & 0.66 & 2.60 & 0.63 & 4.27 & 0.85 & 8.61 & \textbf{0.48} & \textbf{0.99} & 0.24 & 3.58 & 0.76 & 18.02 & \textbf{0.61} & \textbf{0.85} &  \textbf{0.61} &  \textbf{0.86} & \textcolor{blue}{\textbf{0.63}} & \textcolor{blue}{\textbf{0.99}} \\
PointGoal1 & \textbf{0.65} & \textbf{0.95} & \textbf{0.71} & \textbf{0.98} & 0.74 & 1.18 & 0.49 & 1.66 & \textbf{0.43} & \textbf{0.54} & \textbf{0.57} & \textbf{0.35} & 0.64 & 5.38 & \textbf{0.62} & \textbf{0.78} & \textbf{0.44} & \textbf{0.36} & \textcolor{blue}{\textbf{0.68}} & \textcolor{blue}{\textbf{0.79}} \\
PointGoal2 & 0.54 & 1.97 & 0.67 & 3.18 & 0.67 & 3.11 & 0.38 & 1.92 & \textbf{0.29} & \textbf{0.78} & 0.40 & 1.31 & 0.31 & 1.67 & 0.54 & 2.01 & \textbf{0.31} & \textbf{0.59} & 0.61 & 1.78 \\
PointPush1 & \textbf{0.19} & \textbf{0.61} & \textbf{0.33} & \textbf{0.86} & \textbf{0.22} & \textbf{0.79} & \textbf{0.13} & \textbf{0.83} & \textbf{0.13} & \textbf{0.43} & \textbf{0.20} & \textbf{0.83} & \textbf{0.26} & \textbf{0.09} & 0.28 & 1.11 & \textbf{0.15} & \textbf{0.42} & \textcolor{blue}{\textbf{0.33}} & \textcolor{blue}{\textbf{0.95}}\\
PointPush2 & \textbf{0.18} & \textbf{0.91} & \textbf{0.23} & \textbf{0.99} & \textbf{0.16} & \textbf{0.89} & 0.02 & 1.18 & \textbf{0.11} & \textbf{0.80} & 0.11 & 1.04 & \textbf{0.24} & \textbf{0.34} & 0.23 & 1.69 & \textbf{0.15} &  \textbf{0.8} & \textcolor{blue}{\textbf{0.24}} & \textcolor{blue}{\textbf{0.99}}\\
CarCircle1 & 0.72 & 4.39 & 0.73 & 5.25 & 0.76 & 5.46 & 0.70 & 5.72 & 0.37 & 1.38 & 0.02 & 2.29 & 0.69 & 10.52 & 0.49 & 1.78 & 0.52 & 1.85 & 0.57 & 2.14\\
CarCircle2 & 0.76 & 6.44 & 0.72 & 6.58 & 0.74 & 6.28 & 0.77 & 7.99 & 0.54 & 3.38 & 0.44 & 2.69 & 0.63 & 12.78 & 0.61 & 3.45 & 0.59 & 2.33 & 0.62 & 2.46 \\
CarGoal1    &  \textbf{0.39 } & \textbf{0.33 } & \textbf{0.47 } & \textbf{0.78 } & 0.61  & 1.13  &  \textbf{0.35 }& \textbf{0.54 } & \textbf{0.24 }& \textbf{0.28 }& 0.79 & 1.42 & \textbf{0.49 }& \textbf{0.12 }& 0.59 & 1.01 & \textbf{0.38} & \textbf{0.39} & \textcolor{blue}{\textbf{0.61}} & \textcolor{blue}{\textbf{0.99}} \\
CarGoal2    &  0.23  & 1.05  & 0.3   & 1.44  & 0.28  & 1.01  &  \textbf{0.25} & \textbf{0.91}  & \textbf{0.14} & \textbf{0.51} & 0.65 & 3.75 & \textbf{0.06} & \textbf{0.05} & 0.44 & 1.61 & \textbf{0.19} & \textbf{0.52} & 0.48 & 1.62 \\
CarPush1    &  \textbf{0.22}  & \textbf{0.36}  & \textbf{0.23}  & \textbf{0.43}  & \textbf{0.21}  & \textbf{0.54}  &  \textbf{0.23} & \textbf{0.5 }  & \textbf{0.14} & \textbf{0.33} & \textbf{-0.03} & \textbf{0.95} & \textbf{0.28} & \textbf{0.04} & \textbf{0.28} & \textbf{0.60} &  \textbf{0.19} &  \textbf{0.18} & \textcolor{blue}{\textbf{0.34}} & \textcolor{blue}{\textbf{0.63}} \\
CarPush2    &  \textbf{0.14}  &  \textbf{0.9 } & 0.15  & 1.38  & 0.1   & 1.2   & 0.09  & 1.07  & \textbf{0.05} & \textbf{0.45} & 0.24 & 4.25 & \textbf{0.14} & \textbf{0.13} & 0.18 & 2.22 & \textbf{0.08} & \textbf{0.54} & \textcolor{blue}{\textbf{0.18}} & \textcolor{blue}{\textbf{0.98}} \\
SwimmerVel  &  0.49  & 4.72  & 0.48  & 6.58  & 0.3   & 2.33  & 0.63  & 7.58  & 0.51 & 1.07 & 0.13 & 2.66 & \textbf{-0.04} & \textbf{0.00 }& \textbf{0.63} & \textbf{0.30} & 0.55  & 3.21 & \textcolor{blue}{\textbf{0.68}} & \textcolor{blue}{\textbf{0.46}} \\
HopperVel   &  0.65  &  6.39 & 0.78  & 5.02  & 0.34  & 5.86  & 0.13  & 1.51  & \textbf{0.36} & \textbf{0.67} & 0.14 & 2.11 & \textbf{0.17} & \textbf{0.32} & \textbf{0.61} & \textbf{0.22} & \textbf{0.57} & \textbf{0.98} & \textcolor{blue}{\textbf{0.80}} & \textcolor{blue}{\textbf{0.45}} \\
HalfCheeVel &  0.97  &  13.1 & 1.05  & 18.21 & 0.98  & 6.58  & \textbf{0.65}  & \textbf{0.0 }  & \textbf{0.88} & \textbf{0.54} & \textbf{0.29} & \textbf{0.74} & \textbf{0.89} & \textbf{0.00} & \textbf{0.97} & \textbf{0.04} & 0.96 & 2.5 & \textcolor{blue}{\textbf{0.98}} & \textcolor{blue}{\textbf{0.08}} \\
Walker2dVel &  0.79  &  3.88 & \textbf{0.79}  & \textbf{0.17}  & 0.86  &  3.1  & \textbf{0.12}  & \textbf{0.74}  & \textbf{0.79} & \textbf{0.04} & \textbf{0.04} & \textbf{0.21} & \textbf{0.38} & \textbf{0.36} & \textbf{0.79} & \textbf{0.03} & \textbf{0.64} & \textbf{0.06} & \textcolor{blue}{\textbf{0.80}} & \textcolor{blue}{\textbf{0.01}} \\
AntVel      &  0.98  &  3.72 & 1.02  & 4.15  & \textbf{-1.01} &  \textbf{0.0  }& 1.0   & 3.28  & \textbf{0.98 }& \textbf{0.29 }& \textbf{-1.01} & \textbf{0.00 }& \textbf{0.89 }& \textbf{0.00 }& \textbf{0.95} & \textbf{0.18} & \textbf{0.97} & \textbf{0.15} & \textcolor{blue}{\textbf{0.98}} & \textcolor{blue}{\textbf{0.19}} \\
\midrule
\textbf{Average} & 0.55 & 3.40 & 0.58 & 3.59 & 0.43 & 2.77 & 0.45 & 2.91 & \textbf{0.40} & \textbf{0.74} & 0.21 & 1.70 & 0.42 & 3.81 & 0.55 & 1.09 & \textbf{0.46} & \textbf{0.93} & \textcolor{blue}{\textbf{0.59}} & \textcolor{blue}{\textbf{0.95}} \\
\midrule
\multicolumn{21}{l}{\textbf{Bullet Safety Gym:}} \\
BallRun      &  0.6  &  5.08 & 0.76  & 3.91  & -0.47 & 5.03  & 0.59 & 3.52  & 0.27 & 1.46 & 0.22 & 1.27 & \textbf{0.18} & \textbf{0.00} & \textbf{0.37} & \textbf{0.62} & \textbf{0.27} & \textbf{0.49} & \textcolor{blue}{\textbf{0.39}} & \textcolor{blue}{\textbf{0.64}} \\
CarRun       &  \textbf{0.97} &  \textbf{0.33} & \textbf{0.94}  & \textbf{0.15}  & 0.68  & 7.78  & \textbf{0.87} & \textbf{0.0 }  & \textbf{0.94} & \textbf{0.22} & 0.95 & 1.79 & \textbf{0.73} & \textbf{0.04} & \textbf{0.98} & \textbf{0.54} & \textbf{0.97} & \textbf{0.1 }& \textcolor{blue}{\textbf{0.99}} & \textcolor{blue}{\textbf{0.29}} \\
DroneRun     &  0.24 &  2.13 & 0.72  & 5.54  & 0.42  & 2.47  & 0.67 & 4.15  & \textbf{0.28} & \textbf{0.74} & 0.33 & 3.52 & \textbf{0.30} & \textbf{0.16} & \textbf{0.57} & \textbf{0.39} & \textbf{0.55} & \textbf{0.01} & \textcolor{blue}{\textbf{0.58}} & \textcolor{blue}{\textbf{0.00}} \\
AntRun       & 0.7   &  2.93 & 0.76  & 5.11  & \textbf{0.15}  & \textbf{0.73}  & \textbf{0.61} & \textbf{0.94}  & 0.65 & 1.09 & \textbf{0.03} & \textbf{0.02} & \textbf{0.45} & \textbf{0.00} & \textbf{0.62} & \textbf{0.56} & \textbf{0.69} & \textbf{0.76} & \textcolor{blue}{\textbf{0.71}} & \textcolor{blue}{\textbf{0.74}} \\
BallCircle   & 0.74  &  4.71 & 0.69  & 2.36  & 0.86  & 3.09  & 0.7  & 2.61  & \textbf{0.52} & \textbf{0.65} & \textbf{0.64} & \textbf{0.76} & \textbf{0.34} & \textbf{0.00} & \textbf{0.58} & \textbf{0.94} & \textcolor{blue}{\textbf{0.69}} & \textcolor{blue}{\textbf{0.66}} & \textbf{0.59} & \textbf{0.95} \\
CarCircle    & 0.58  &  3.74 & 0.63  & 1.89  & 0.74  & 2.18  & 0.49 & 3.14  & \textbf{0.50} & \textbf{0.84} & \textbf{0.71} & \textbf{0.33} & \textbf{0.40} & \textbf{0.03} & \textbf{0.72} & \textbf{0.63} & \textbf{0.61} & \textbf{0.87} & \textcolor{blue}{\textbf{0.74}} & \textcolor{blue}{\textbf{0.65}} \\
DroneCircle  & 0.72  &  3.03 & 0.8   & 3.07  & 0.78  & 3.68  & 0.26 & 1.02  & \textbf{0.56} & \textbf{0.57} & -0.22 & 1.28 & \textbf{0.48} & \textbf{0.00} & \textbf{0.61} & \textbf{0.96} & \textbf{0.6}& \textbf{0.66} & \textcolor{blue}{\textbf{0.62}} & \textcolor{blue}{\textbf{0.95}} \\
AntCircle    & 0.58  &  4.9  & 0.58  & 2.87  & 0.65  & 5.48  & 0.17 & 5.04  & \textbf{0.40} & \textbf{0.96} & \textbf{0.00} & \textbf{0.20} & \textbf{0.20} & \textbf{0.00} & 0.45 & 1.54 &  \textcolor{blue}{\textbf{0.49}} & \textcolor{blue}{\textbf{0.91}} & 0.45 & 1.19 \\
\midrule
\textbf{Average} & 0.64 & 3.36 & 0.73 & 3.11 & 0.48 & 3.81 & 0.54 & 2.55 & \textbf{0.52} & \textbf{0.82} & 0.33 & 1.15 & \textbf{0.39} & \textbf{0.03} & \textbf{0.61} & \textbf{0.77} & \textbf{0.61} & \textbf{0.56} & \textcolor{blue}{\textbf{0.63}} & \textcolor{blue}{\textbf{0.68}} \\
\midrule
\multicolumn{21}{l}{\textbf{MetaDrive:}} \\
EasySparse & 0.17 & 1.54 & 0.78 & 5.01 & \textbf{0.11} & \textbf{0.86} & 0.96 & 5.44 & \textbf{0.11} & \textbf{0.21} & \textbf{-0.06} & \textbf{0.07 }& \textbf{0.38} & \textbf{0.15} & \textbf{0.13} & \textbf{0.25} & \textbf{0.45} &  \textbf{0.52} & \textcolor{blue}{\textbf{0.82}} & \textcolor{blue}{\textbf{0.96}} \\
EasyMean & 0.43 & 2.82 & 0.71 & 3.44 & \textbf{0.08} & \textbf{0.86} & 0.66 & 3.97 & \textbf{0.04} & \textbf{0.29} & \textbf{-0.07} & \textbf{0.07 }& \textbf{0.38} & \textbf{0.08} & \textbf{0.47} & \textbf{0.62} &  \textbf{0.4}& \textbf{0.47} & \textcolor{blue}{\textbf{0.74}} & \textcolor{blue}{\textbf{0.69}} \\
EasyDense & 0.27 & 1.94 & \textbf{0.26} & \textbf{0.47} & \textbf{0.02} & \textbf{0.41} & 0.50 & 2.54 & \textbf{0.11} & \textbf{0.14} & \textbf{-0.06} & \textbf{0.03 }& \textbf{0.36} & \textbf{0.08} & \textbf{0.36} & \textbf{0.89} &  \textbf{0.37} &  \textbf{0.41} & \textcolor{blue}{\textbf{0.69}} & \textcolor{blue}{\textbf{0.69}} \\
MediumSparse & 0.83 & 3.34 & 0.44 & 1.16 & \textbf{-0.03} & \textbf{0.17} & 0.71 & 2.49 & \textbf{0.33} & \textbf{0.30} & \textbf{-0.08} & \textbf{0.07 }& \textbf{0.42} & \textbf{0.07} & 0.73 & 1.27 & \textbf{0.8 }& \textbf{0.53} & \textcolor{blue}{\textbf{0.91}} & \textcolor{blue}{\textbf{0.98}} \\
MediumMean & 0.77 & 2.53 & 0.78 & 1.53 & \textbf{0.00} & \textbf{0.34} & 0.76 & 2.05 & \textbf{0.31} & \textbf{0.21} & \textbf{-0.08} & \textbf{0.05 }& \textbf{0.39} & \textbf{0.02} & \textbf{0.45} & \textbf{0.89} & \textbf{0.74} & \textbf{0.58} & \textcolor{blue}{\textbf{0.86}} & \textcolor{blue}{\textbf{0.52}} \\
MediumDense & 0.45 & 1.47 & 0.58 & 1.89 & \textbf{0.01} & \textbf{0.28} & 0.69 & 2.24 & \textbf{0.24} & \textbf{0.17} & \textbf{-0.07} & \textbf{0.07 }& \textbf{0.49} & \textbf{0.12} & 0.82 & 2.51 & \textbf{0.75} &  \textbf{0.58} & \textcolor{blue}{\textbf{0.89}} & \textcolor{blue}{\textbf{0.52}} \\
HardSparse & 0.42 & 1.80 & 0.50 & 1.02 & \textbf{0.01} & \textbf{0.16} & 0.37 & 2.05 & 0.17 & 3.25 & \textbf{-0.05} & \textbf{0.06 }& \textbf{0.30} & \textbf{0.00} & \textbf{0.22} & \textbf{0.45} & \textbf{0.43} & \textbf{0.64} & \textcolor{blue}{\textbf{0.44}} & \textcolor{blue}{\textbf{0.64}} \\
HardMean & 0.20 & 1.77 & 0.47 & 2.56 & \textbf{0.00} & \textbf{0.21} & 0.32 & 2.47 & \textbf{0.13} & \textbf{0.40} & \textbf{-0.05} & \textbf{0.06 }& \textbf{0.26} & \textbf{0.09} & \textbf{0.31} & \textbf{0.96} &  \textbf{0.45} & \textbf{0.67} & \textcolor{blue}{\textbf{0.45}} & \textcolor{blue}{\textbf{0.66}} \\
HardDense & 0.20 & 1.33 & 0.35 & 1.40 & \textbf{0.02} & \textbf{0.26} & 0.24 & 1.68 & \textbf{0.15} & \textbf{0.22} & \textbf{-0.04} & \textbf{0.08 }& \textbf{0.30} & \textbf{0.10} & \textbf{0.12} & \textbf{0.69} & \textbf{0.35} & \textbf{0.5 }& \textcolor{blue}{\textbf{0.42}} & \textcolor{blue}{\textbf{0.76}} \\
\midrule
\textbf{Average} & 0.42 & 2.06 & 0.54 & 2.05 & \textbf{0.02} & \textbf{0.39} & 0.58 & 2.77 & \textbf{0.18} & \textbf{0.58} & \textbf{-0.0}6 & \textbf{0.06} & \textbf{0.36} & \textbf{0.08} & \textbf{0.40} & \textbf{0.95} & \textbf{0.53} & \textbf{0.54} & \textcolor{blue}{\textbf{0.69}} & \textcolor{blue}{\textbf{0.71}} \\
\bottomrule
\end{tabular}
}
\end{table}

% These findings demonstrate that incorporating return-coverage rebalancing into constrained sequence modeling enables the agent to more effectively navigate reward-cost trade-offs under diverse safety requirements.

\subsection{Main Results for \algcb~on DSRL Benchmark}\label{sec:exp1}

The experimental results of \algcb~and all baselines on the DSRL benchmark are summarized in Table~\ref{tab:main-result}. Overall, \algcb~achieves consistently strong performance across SafetyGym, BulletSafetyGym, and MetaDrive, attaining the best return-cost performance on the majority of tasks.
Interpreting Table~\ref{tab:main-result} requires jointly considering safety and task performance. 
We first identify methods that satisfy the cost constant, indicated by \textbf{bold} entries where the averaged normalized cost is below~$1$. Among these safe methods, we primarily compare the averaged normalized return, with the highest safe return highlighted in \textcolor{blue}{blue}. Lower cost among already-safe methods is reported but treated as secondary, since overly conservative behavior can reduce cost without translating into meaningful task progress. 
This evaluation criterion aligns with the objective of constrained offline RL: satisfying safety requirements while preserving as much task performance as possible.

BC simply imitates the behavior in the offline dataset without any cost awareness, and can be viewed as an approximation of the behavior policy $\beta$. 
As shown in Table~\ref{tab:main-result}, such naive imitation does not provide meaningful safety guarantees, and often yields costs well above the constant. BC-Safe, in contrast, explicitly restricts training to trajectories that satisfy the cost constraint, and consequently achieves normalized costs below~1 on most tasks. 
However, this safety is obtained at the expense of substantially reduced returns, particularly on more challenging velocity-control and circle tasks where high-return behavior is relatively scarce in the dataset. 
Notably, \algcb~consistently attains higher returns than BC-Safe across all tasks, while maintaining competitive costs in most cases, leading to a markedly better return--cost trade-off. 
This comparison highlights an inherent limitation of purely supervised imitation-style methods under imbalanced data: imitating only a small “safe” subset may enforce constraint satisfaction, but the resulting loss of coverage over high-return regions severely limits achievable performance and generalization.

Value-based offline RL baselines with constraint handling (BCQ-Lag, BEAR-Lag, COptiDICE, and CPQ) often struggle to deliver a reliable return--cost trade-off across domains. 
On SafetyGym and BulletSafetyGym, these methods frequently incur normalized costs well above the safety constraint on a nontrivial portion of tasks, even when their returns are relatively competitive. 
In contrast, on MetaDrive, several of them (BEAR-Lag and CPQ) exhibit overly conservative behavior with markedly reduced returns despite achieving low costs. 
Taken together, these results suggest that simply incorporating scalar penalties, Lagrangian relaxations, or distribution-matching objectives into otherwise unconstrained offline RL pipelines does not consistently provide: the resulting trade-off is highly sensitive to the environment and the underlying data distribution, and does not robustly balance safety satisfaction with task performance.
By comparison, \algcb~achieves substantially more consistent return--cost trade-offs across SafetyGym, BulletSafetyGym, and MetaDrive, attaining strong returns while keeping costs close to or below the safety constraint on the vast majority of tasks.

FISOR and TraC are strong offline safe baselines that explicitly target safety-critical offline learning and often achieve low costs in practice. 
This pattern is evident from Table~\ref{tab:main-result}. 
On BulletSafetyGym and MetaDrive, FISOR attains extremely low average costs (0.03 and 0.08) but yields notably lower average returns (0.39 and 0.36) than \algcb~(0.63 and 0.69). 
On SafetyGym, FISOR is less reliable, with an average cost of 3.81 despite a moderate return of 0.42. 
TraC maintains relatively low costs across domains (average costs of 0.93/0.56/0.54 on SafetyGym/BulletSafetyGym/MetaDrive), yet its average returns (0.46/0.61/0.53) consistently trail those of \algcb~(0.59/0.63/0.69). 
Overall, while these methods provide competitive safety-oriented baselines, \algcb~achieves a more favorable return--cost trade-off across all three domains, delivering higher returns while keeping costs close to or below the safety constraint on the majority of tasks.

The comparison with CDT is particularly informative. Both CDT and \algcb~are CTG-conditioned CSM methods, so their policies can be evaluated across multiple cost thresholds in a zero-shot manner. 
As shown in Table~\ref{tab:main-result}, CDT tends to yield costs that concentrate around the safety constraint, and it even slightly exceeds the threshold on SafetyGym on average (cost $=1.09$). 
In contrast, \algcb~achieves lower costs than CDT across all three domains (1.09$\rightarrow$0.95 on SafetyGym, 0.77$\rightarrow$0.68 on BulletSafetyGym, and 0.95$\rightarrow$0.71 on MetaDrive), while also improving average returns (0.55$\rightarrow$0.59, 0.61$\rightarrow$0.63, and 0.40$\rightarrow$0.69, respectively).
This gap is consistent with the fact that CDT can be viewed as a degenerate variant of \algcb~that relies on naive CTG-conditioned maximum-likelihood training alone. 
By augmenting CDT with (i) a Lagrangian-style cost penalty with an auto-adaptive coefficient, (ii) reward--cost-aware trajectory-level reweighting, and (iii) Q-value regularization, \algcb~more effectively suppresses behaviors that lead to constraint violations while still encouraging high-return actions. 
As a result, \algcb~achieves a more favorable return--cost trade-off than CDT, demonstrating that principled regularization and reweighting are essential for realizing the full potential of CSM backbones in CMDPs.

Overall, the experimental results show that \algcb~achieves consistently strong return--cost trade-offs on the DSRL benchmark. Compared with representative baselines, \algcb~more reliably attains high normalized returns while satisfying the cost constraint on the vast majority of tasks, indicating robust behavior across diverse environments and data regimes. Moreover, \algcb~supports one-policy evaluation across multiple cost thresholds: a single trained policy can be deployed in a zero-shot manner by adjusting the target CTG, accommodating different safety requirements at inference time. Taken together, these comparisons demonstrate the effectiveness of \algcb~across multiple methodological baselines and suggest that it is a competitive approach with strong potential to advance the state-of-the-art in offline safe RL.

\subsection{Ablation Studies}\label{sec:exp3}
\newcommand{\ablwq}{WQDT}
\newcommand{\ablwc}{WCDT}
\newcommand{\ablqc}{QCDT}

\subsubsection{Component-level Ablation}

To empirically investigate the contribution of each individual component in the proposed \algcb~under safety-constrained settings, we compare \algcb~with its ablation variants on representative SafetyGym tasks. The ablations target four key aspects: \emph{CSM-based architecture}, \emph{trajectory-level weighting}, \emph{Q-value guidance}, and \emph{Lagrangian-style cumulative cost penalty}, whose configurations are summarized in Table~\ref{tab:ablation_components}. In this table, “CSM-based architecture" indicates whether the policy is implemented as a Decision-Transformer-style conditional sequence model, “Trajectory-level weighting" denotes the use of our data-weighting mechanism, “Q-value guidance" corresponds to value-based guidance, and “Cumulative-cost penalty" reflects the presence of the cost-penalization term in the objective. BC and BC-Safe serve as imitation-learning baselines without a CSM backbone; BC-Safe differs from BC only in that it applies trajectory-level weighting to favor safer trajectories. CDT is the vanilla CSM-based baseline without any of the newly introduced components. Building on CDT, \ablwq~enables both trajectory-level weighting and Q-value guidance, \ablwc~combines trajectory-level weighting with the cumulative-cost penalty, and \ablqc~uses Q-value guidance together with the cumulative-cost penalty but no weighting. Finally, \algcb~activates all four components and thus represents the full version of our proposed method.

Table~\ref{tab:ablation_dsrl} summarizes the component-level ablation results on three representative SafetyGym velocity-constrained tasks.
We first observe that plain behavior cloning (BC) consistently yields unsafe policies, exhibiting large normalized costs across all tasks (e.g., costs well above $1$), which confirms that these datasets are challenging and that naive imitation of the behavior data does not suffice for constraint satisfaction.
Introducing safety-aware trajectory reweighting in BC-Safe substantially reduces cost compared with BC, but this improvement often comes with a notable return drop (e.g., on \texttt{HopperVel}), reflecting the typical conservatism induced by prioritizing low-cost behavior without an explicit mechanism to recover high-return actions.

A key takeaway is that \textbf{a CSM backbone provides a stronger and more controllable foundation for offline safe RL}.
Although CDT is also trained via maximum-likelihood action modeling, it consistently achieves a \emph{better} return--cost trade-off than BC-Safe, attaining both higher return and lower cost across all tasks.
This suggests that the target-conditioned sequence-modeling form can better accommodate constraint variation at evaluation time, and that leveraging trajectory context together with RTG/CTG conditioning yields more effective trade-off control than non-CSM imitation baselines.
Notably, all \algcb~variants built on top of the CSM architecture inherit this advantage, further highlighting the importance of CSM as the policy backbone in safety-constrained offline learning.

The ablations also confirm that \textbf{each proposed component contributes meaningfully and the full combination is necessary for the best trade-off}.
Compared with CDT, all partial variants (\ablwq, \ablwc, \ablqc) that activate any two of the three mechanisms (weighting, Q-guidance, and cost penalty) improve (or keep) return while remaining within a comparable safety range, indicating that each ingredient provides non-trivial gains over naive CTG-conditioned training.
However, these gains are consistently smaller than those of the full method: \algcb~achieves the best overall return--cost trade-off on all tasks, and removing any single component leads to performance degradation, demonstrating that the components are complementary rather than redundant.

A more fine-grained comparison further clarifies the role of each ingredient.
First, \ablwq~ consistently increases return over CDT, but it also increases cost across all tasks, which matches our design intuition: trajectory-level weighting and Q-value guidance primarily mitigate overly conservative behavior and steer learning toward higher-return regions, typically at the expense of additional cost.
Crucially, \algcb~introduces the Lagrangian-style cumulative-cost penalty on top of these return-seeking signals, and achieves lower cost than \ablwq~without sacrificing its return gains, indicating that the penalty term provides an effective counterbalance that stabilizes constraint satisfaction while remaining compatible with the other mechanisms.
Second, \ablwc~ and \ablqc~ both combine one return-improving signal with the cost penalty; the resulting policies generally outperform CDT in return but still exhibit higher costs than the full method, suggesting that using only a single return-seeking mechanism is insufficient to fully resolve the conservatism--violation tension, and that the best trade-off emerges only when both weighting and Q-guidance are jointly leveraged under the cost-penalized objective.

Overall, the ablation study provides clear evidence for the intended functional roles of the components:
the Lagrangian-style penalty is essential for reducing cost and stabilizing constraint satisfaction, whereas trajectory-level weighting and Q-value guidance are key to improving return by avoiding overly conservative solutions.
Importantly, \algcb~achieves the most favorable return--cost trade-off only when all components are enabled, validating the necessity and complementarity of the proposed design choices.

\begin{table}[t]
\centering
\small
\caption{Component-level breakdown across ablation variants, where a $\checkmark$ indicates that the corresponding component is enabled and a $\times$ that it is disabled.}
\label{tab:ablation_components}
\begin{tabular}{l|ccccccc}
\toprule
\textbf{Component} &
\textbf{BC} &
\textbf{BC-Safe} &
\textbf{CDT} &
\textbf{\ablwq} &
\textbf{\ablwc} &
\textbf{QCDT} &
\textbf{\algcb} \\
\midrule
CSM-based architecture      & $\times$ & $\times$ & $\checkmark$ & $\checkmark$ & $\checkmark$ & $\checkmark$ & $\checkmark$ \\
Trajectory-level weighting  & $\times$ & $\checkmark$ & $\times$ & $\checkmark$ & $\checkmark$ & $\times$ & $\checkmark$ \\
Q-value guidance            & $\times$ & $\times$ & $\times$ & $\checkmark$ & $\times$ & $\checkmark$ & $\checkmark$ \\
Cumulative-cost penalty     & $\times$ & $\times$ & $\times$ & $\times$ & $\checkmark$ & $\checkmark$ & $\checkmark$ \\
\bottomrule
\end{tabular}
\end{table}

\begin{table}[t]
\centering
\small
\caption{Performance comparison of ablation variants across DSRL tasks. For each method, we report normalized return (ret, higher is better) and normalized cost (cost, lower is better).}
\label{tab:ablation_dsrl}
\resizebox{\textwidth}{!}{
\begin{tabular}{l|cc|cc|cc|cc|cc|cc|cc}
\toprule
\multirow{2}{*}{\textbf{Task}} &
\multicolumn{2}{c}{\textbf{BC}} &
\multicolumn{2}{c}{\textbf{BC-Safe}} &
\multicolumn{2}{c}{\textbf{CDT}} &
\multicolumn{2}{c}{\textbf{WQDT}} &
\multicolumn{2}{c}{\textbf{WCDT}} &
\multicolumn{2}{c}{\textbf{QCDT}} &
\multicolumn{2}{c}{\textbf{\algcb}} \\
\cmidrule(lr){2-3}
\cmidrule(lr){4-5}
\cmidrule(lr){6-7}
\cmidrule(lr){8-9}
\cmidrule(lr){10-11}
\cmidrule(lr){12-13}
\cmidrule(lr){14-15}
& 
ret $\uparrow$ & cost $\downarrow$
& ret $\uparrow$ & cost $\downarrow$
& ret $\uparrow$ & cost $\downarrow$
& ret $\uparrow$ & cost $\downarrow$
& ret $\uparrow$ & cost $\downarrow$
& ret $\uparrow$ & cost $\downarrow$
& ret $\uparrow$ & cost $\downarrow$ \\
\midrule
SwimmerVel      & 0.49 & 4.72 & 0.51 & 1.07                   & \textbf{0.63} & \textbf{0.30} & \textcolor{blue}{\textbf{0.68}} & \textcolor{blue}{\textbf{0.82}} & \textbf{0.66} & \textbf{0.49} & \textbf{0.67} & \textbf{0.76} & \textcolor{blue}{\textbf{0.68}} & \textcolor{blue}{\textbf{0.46}} \\
HopperVel       & 0.65 & 6.39 & \textbf{0.36} & \textbf{0.67} & \textbf{0.61} & \textbf{0.22} & \textcolor{blue}{\textbf{0.80}} & \textcolor{blue}{\textbf{0.54}} & \textbf{0.77} & \textbf{0.52} & \textbf{0.79} & \textbf{0.78} & \textcolor{blue}{\textbf{0.80}} & \textcolor{blue}{\textbf{0.45}} \\
Walker2dVel     & 0.79 & 3.88 & \textbf{0.79} & \textbf{0.04} & \textbf{0.79} & \textbf{0.03} & \textcolor{blue}{\textbf{0.80}} & \textcolor{blue}{\textbf{0.08}} & \textbf{0.78} & \textbf{0.01} & \textbf{0.79} & \textbf{0.04} & \textcolor{blue}{\textbf{0.80}} & \textcolor{blue}{\textbf{0.01}} \\
\bottomrule
\end{tabular}
}
\end{table}

\subsubsection{Robustness Evaluation on DSRL Dataset Variants}\label{sec:exp2}

To further examine the robustness of \algcb~under more demanding data regimes by evaluating it on a set of dataset variants derived from the original DSRL benchmark. All variants are constructed by retaining a subset of trajectories from each original dataset to induce (\romannumeral1) a lower proportion of expert-level trajectories and (\romannumeral2) more imbalanced return distributions. The corresponding results are reported in Table~\ref{tab:worse_data_comparison}, and Table~\ref{tab:imbalance_data_comparison}, respectively. For each environment, we order the retention ratios from larger to smaller values, so that the reported settings become progressively more challenging as the available data decreases (or becomes more skewed).

\begin{table}[t]
\centering
\small
\caption{Performance comparison on datasets constructed by retaining only the bottom-$\rho\%$ trajectories (ranked by return) from each original dataset (i.e., discarding the top-$(100-\rho)\%$ expert trajectories).}
\label{tab:worse_data_comparison}
\resizebox{0.5\textwidth}{!}{
\begin{tabular}{ll|cc|cc|cc}
\toprule
\multirow{2}{*}{\textbf{Task}} & \multirow{2}{*}{$\rho\%$} 
& \multicolumn{2}{c|}{\textbf{BC}}
& \multicolumn{2}{c|}{\textbf{CDT}}
& \multicolumn{2}{c}{\textbf{\algcb}} \\
\cmidrule(lr){3-4}\cmidrule(lr){5-6}\cmidrule(lr){7-8}
& 
& rew $\uparrow$ & cost $\downarrow$
& rew $\uparrow$ & cost $\downarrow$
& rew $\uparrow$ & cost $\downarrow$ \\
\midrule
\multirow{4}{*}{SwimmweVel} & $100\%$ & 0.49 & 4.72 & \textbf{0.63} & \textbf{0.30} & \textcolor{blue}{\textbf{0.68}} & \textcolor{blue}{\textbf{0.46}} \\
                      & $80\%$ & 0.49 & 4.66 & \textbf{0.54} & \textbf{0.63} & \textcolor{blue}{\textbf{0.58}} & \textcolor{blue}{\textbf{0.35}} \\
                      & $70\%$ & 0.38 & 3.30 & \textbf{0.47} & \textbf{0.69} & \textcolor{blue}{\textbf{0.52}} & \textcolor{blue}{\textbf{0.42}} \\
                      & $60\%$ & \textbf{0.42} & \textbf{0.85} & \textcolor{blue}{\textbf{0.63}} & \textcolor{blue}{\textbf{0.76}} & \textbf{0.43} & \textbf{0.41} \\
\midrule
\multirow{4}{*}{HopperVel} & $100\%$ & 0.65 & 6.39 & \textbf{0.61} & \textbf{0.22} & \textcolor{blue}{\textbf{0.80}} & \textcolor{blue}{\textbf{0.45}} \\
                      & $80\%$ & 0.66 & 4.38 & \textbf{0.61} & \textbf{0.35} & \textcolor{blue}{\textbf{0.63}} & \textcolor{blue}{\textbf{0.43}} \\
                      & $70\%$ & 0.59 & 5.13 & \textbf{0.27} & \textbf{0.64} & \textcolor{blue}{\textbf{0.52}} & \textcolor{blue}{\textbf{0.42}} \\
                      & $60\%$ & 0.56 & 5.39 & \textbf{0.27} & \textbf{0.73} & \textcolor{blue}{\textbf{0.46}} & \textcolor{blue}{\textbf{0.74}} \\
\midrule
\multirow{4}{*}{HalfCheetahVel} & $100\%$ & 0.97 & 13.1 & \textbf{0.97} & \textbf{0.04} & \textcolor{blue}{\textbf{0.98}} & \textcolor{blue}{\textbf{0.08}} \\
                      & $80\%$ & 0.89 & 3.26 & \textbf{0.93} & \textbf{0.48} & \textcolor{blue}{\textbf{0.93}} & \textcolor{blue}{\textbf{0.33}} \\
                      & $70\%$ & 0.89 & 4.64 & \textbf{0.91} & \textbf{0.65} & \textcolor{blue}{\textbf{0.89}} & \textcolor{blue}{\textbf{0.22}} \\
                      & $60\%$ & 0.89 & 4.81 & 0.86 & 1.64 & \textcolor{blue}{\textbf{0.86}} & \textcolor{blue}{\textbf{0.75}} \\
\midrule
\multirow{4}{*}{Walker2dVel} & $100\%$ & 0.79 & 3.88 & \textbf{0.79} & \textbf{0.03} & \textcolor{blue}{\textbf{0.80}} & \textcolor{blue}{\textbf{0.01}} \\
                      & $80\%$ & 0.79 & 3.54  & \textbf{0.74} & \textbf{0.48} & \textcolor{blue}{\textbf{0.79}} & \textcolor{blue}{\textbf{0.01}} \\
                      & $70\%$ & 0.79 & 3.20  & \textbf{0.72} & \textbf{0.86} & \textcolor{blue}{\textbf{0.79}} & \textcolor{blue}{\textbf{0.01}} \\
                      & $60\%$ & 0.77 & 2.00  & \textbf{0.78} & \textbf{0.00} & \textcolor{blue}{\textbf{0.79}} & \textcolor{blue}{\textbf{0.01}} \\
\midrule
\multirow{4}{*}{AntVel} & $100\%$ & 0.98 & 3.72 & \textbf{0.95} & \textbf{0.18} & \textcolor{blue}{\textbf{0.98}} & \textcolor{blue}{\textbf{0.19}} \\
                      & $80\%$ & 0.96 & 5.55 & \textbf{0.94} & \textbf{0.48} & \textcolor{blue}{\textbf{0.96}} & \textcolor{blue}{\textbf{0.48}} \\
                      & $70\%$ & 0.98 & 10.16& \textbf{0.94} & \textbf{0.43} & \textcolor{blue}{\textbf{0.95}} & \textcolor{blue}{\textbf{0.36}} \\
                      & $60\%$ & 0.89 & 2.91 & \textbf{0.92} & \textbf{0.33} & \textcolor{blue}{\textbf{0.94}} & \textcolor{blue}{\textbf{0.42}} \\
\bottomrule
\end{tabular}
}
\end{table}

\begin{table}[t]
\centering
\small
\caption{Performance comparison on more imbalanced datasets constructed by retaining only the top-$\rho\%$ and bottom-$\rho\%$ trajectories (ranked by return) from each original dataset.}
\label{tab:imbalance_data_comparison}
\resizebox{0.5\textwidth}{!}{
\begin{tabular}{ll|cc|cc|cc}
\toprule
\multirow{2}{*}{\textbf{Task}} & \multirow{2}{*}{$\rho\%$} 
& \multicolumn{2}{c|}{\textbf{BC}}
& \multicolumn{2}{c|}{\textbf{CDT}}
& \multicolumn{2}{c}{\textbf{\algcb}} \\
\cmidrule(lr){3-4}\cmidrule(lr){5-6}\cmidrule(lr){7-8}
& 
& rew $\uparrow$ & cost $\downarrow$
& rew $\uparrow$ & cost $\downarrow$
& rew $\uparrow$ & cost $\downarrow$ \\
\midrule
\multirow{4}{*}{SwimmweVel} & $100\%$ & 0.49 & 4.72 & \textbf{0.63} & \textbf{0.30} & \textcolor{blue}{\textbf{0.68}} & \textcolor{blue}{\textbf{0.46}} \\
                      & $30\%$ & 0.64 & 4.36 & \textbf{0.67} & \textbf{0.63} & \textcolor{blue}{\textbf{0.68}} & \textcolor{blue}{\textbf{0.55}} \\
                      & $20\%$ & 0.20 & 2.74 & \textbf{0.64} & \textbf{0.73} & \textcolor{blue}{\textbf{0.68}} & \textcolor{blue}{\textbf{0.77}} \\
                      & $10\%$ & 0.11 & 2.67 & \textbf{0.64} & \textbf{0.84} & \textcolor{blue}{\textbf{0.68}} & \textcolor{blue}{\textbf{0.94}} \\
\midrule
\multirow{4}{*}{HopperVel} & $100\%$ & 0.65 & 6.39 & \textbf{0.61} & \textbf{0.22} & \textcolor{blue}{\textbf{0.80}} & \textcolor{blue}{\textbf{0.45}} \\
                      & $30\%$ & 0.33 & 5.63 & \textbf{0.60} & \textbf{0.43} & \textcolor{blue}{\textbf{0.70}} & \textcolor{blue}{\textbf{0.62}} \\
                      & $20\%$ & 0.17 & 3.65 & \textbf{0.24} & \textbf{1.50} & \textcolor{blue}{\textbf{0.66}} & \textcolor{blue}{\textbf{0.86}} \\
                      & $10\%$ & 0.53 & 13.82& \textbf{0.37} & \textbf{0.27} & \textcolor{blue}{\textbf{0.63}} & \textcolor{blue}{\textbf{0.66}} \\
\midrule
\multirow{4}{*}{HalfCheetahVel} & $100\%$ & 0.97 & 13.1 & \textbf{0.97} & \textbf{0.04} & \textcolor{blue}{\textbf{0.98}} & \textcolor{blue}{\textbf{0.08}} \\
                      & $30\%$ & 0.83 & 2.78 & \textbf{0.97} & \textbf{0.10} & \textcolor{blue}{\textbf{0.98}} & \textcolor{blue}{\textbf{0.13}} \\
                      & $20\%$ & 0.94 & 4.84 & \textbf{0.96} & \textbf{0.08} & \textcolor{blue}{\textbf{0.98}} & \textcolor{blue}{\textbf{0.16}} \\
                      & $10\%$ & 0.73 & 13.25& \textbf{0.94} & \textbf{0.01} & \textcolor{blue}{\textbf{0.99}} & \textcolor{blue}{\textbf{0.22}} \\
\midrule
\multirow{4}{*}{Walker2dVel} & $100\%$ & 0.79 & 3.88 & \textbf{0.79} & \textbf{0.03} & \textcolor{blue}{\textbf{0.80}} & \textcolor{blue}{\textbf{0.01}} \\
                      & $30\%$ & 0.83 & 3.00 & \textbf{0.80} & \textbf{0.02} & \textcolor{blue}{\textbf{0.80}} & \textcolor{blue}{\textbf{0.01}} \\
                      & $20\%$ & 0.81 & 2.73 & \textbf{0.76} & \textbf{0.03} & \textcolor{blue}{\textbf{0.79}} & \textcolor{blue}{\textbf{0.01}} \\
                      & $10\%$ & 0.80 & 2.41 & \textcolor{blue}{\textbf{0.73}} & \textcolor{blue}{\textbf{0.02}} & \textcolor{blue}{\textbf{0.73}} & \textcolor{blue}{\textbf{0.56}} \\
\midrule
\multirow{4}{*}{AntVel} & $100\%$ & 0.98 & 3.72 & \textbf{0.95} & \textbf{0.18} & \textcolor{blue}{\textbf{0.98}} & \textcolor{blue}{\textbf{0.19}} \\
                      & $30\%$ & 0.98 & 5.11 & \textbf{0.98} & \textbf{0.41} & \textcolor{blue}{\textbf{0.99}} & \textcolor{blue}{\textbf{0.71}} \\
                      & $20\%$ & 0.97 & 1.01 & \textbf{0.98} & \textbf{0.39} & \textcolor{blue}{\textbf{0.99}} & \textcolor{blue}{\textbf{0.86}} \\
                      & $10\%$ & 0.95 & 1.88 & \textbf{0.98} & \textbf{0.40} & \textcolor{blue}{\textbf{0.98}} & \textcolor{blue}{\textbf{0.75}} \\
\bottomrule
\end{tabular}
}
\end{table}

\textbf{Performance under lower expert-data coverage.}
% Table~\ref{tab:worse_data_comparison} reports results on degraded datasets constructed by discarding the top-$(100-\rho)\%$ high-return trajectories, which directly reduces expert-level coverage and makes it substantially harder to recover favorable return--cost trade-offs from offline data. 
Table~\ref{tab:worse_data_comparison} reports results on degraded datasets constructed by retaining only the bottom-$\rho\%$ low-return trajectories (ranked by return), which directly reduces expert-level coverage and makes it substantially harder to recover favorable return--cost trade-offs from offline data.
Across the five velocity-constrained locomotion tasks (\texttt{SwimmweVel}, \texttt{HopperVel}, \texttt{HalfCheetahVel}, \texttt{Walker2dVel}, and \texttt{AntVel}), \algcb~remains consistently competitive as $\rho$ decreases from $100\%$ to $60\%$. 
In most settings, \algcb~achieves the highest reward while maintaining costs comparable to CDT, whereas BC often incurs noticeably larger cumulative costs. 
For example, on \texttt{HopperVel} and \texttt{AntVel}, \algcb~preserves strong reward across all retention ratios and avoids the severe cost explosion observed for BC, indicating that the Lagrangian-style penalty provides an effective bias toward constraint satisfaction even when safe expert trajectories are removed. 
Moreover, \algcb~often improves reward over CDT at similar cost levels (e.g., $\rho=100\%$ and $80\%$ on multiple tasks), suggesting that the proposed trajectory-level reweighting and Q-value regularization can still exploit informative sub-optimal data to enhance performance under reduced expert coverage. 

\textbf{Performance under imbalanced data distributions.}
Table~\ref{tab:imbalance_data_comparison} evaluates a complementary stress test where the dataset is made increasingly imbalanced by retaining only the top-$\rho\%$ and bottom-$\rho\%$ trajectories. 
This construction amplifies distribution skew and can exacerbate the tendency of likelihood-trained policies to overfit dominant modes, making robust trade-off control more challenging as $\rho$ decreases. 
Overall, \algcb~exhibits strong robustness across tasks when $\rho$ is reduced from $30\%$ to $10\%$: it consistently achieves the highest reward, and it remains competitive in cumulative cost relative to CDT, while BC typically suffers substantial reward degradation and often very large costs. 
Notably, on \texttt{HalfCheetahVel} and \texttt{SwimmweVel}, \algcb~maintains near-saturated normalized returns even at $\rho=10\%$, demonstrating that its training recipe can still leverage the retained high-quality trajectories without being derailed by the low-return tail. 
At the same time, costs can increase under the most extreme imbalance on some tasks (e.g., \texttt{SwimmweVel} and \texttt{HopperVel} at small $\rho$), reflecting the intrinsic difficulty of satisfying constraints when the dataset provides limited coverage of safe behaviors under heavily skewed mixtures; nevertheless, \algcb~typically delivers a better overall return--cost trade-off than both CDT and BC. 
These results suggest that \algcb~is less sensitive to return-distribution skew and can better tolerate highly imbalanced offline datasets, which aligns with our goal of robust zero-shot safe control under limited and biased coverage.

%%%%%%%%%%%%%%%%%

\section{Conclusion}\label{sec:conclusion}

This work studies offline safe reinforcement learning through the lens of conditional sequence modeling (CSM) and its zero-shot controllability.
We analyze the discrepancy between the specified RTG/CTG signals and the resulting expected return and cumulative cost induced by a conditioned CSM policy, and show that this mismatch admits an upper bound scaling as $O\left(\varepsilon\left(\tfrac{1}{\alpha_F}+2\right)H^2\right)$, where $\alpha_F$ characterizes the joint return--cost coverage of the target profile in the offline dataset.
This characterization clarifies when RTG/CTG conditioning can serve as a reliable mechanism for safe decision making and highlights the fundamental role of data coverage in achieving stable return--cost trade-offs.
Guided by these insights, we propose \algcb, a CSM-based offline safe RL algorithm that integrates a Lagrangian-style cumulative-cost penalty with an auto-adaptive dual-ascent update of the penalty coefficient, together with a reward--cost-aware trajectory-level reweighting mechanism and Q-value regularization.
These components decouple policy learning from any single pre-specified cost threshold while mitigating overly conservative behavior, enabling a single trained policy to support zero-shot deployment across varying cost constraints.
Extensive experiments on the DSRL benchmark under multiple evaluation thresholds demonstrate that \algcb~consistently improves return--cost trade-offs over representative baselines, and suggest that coverage-aware training objectives are a promising direction for advancing CSM-based safe reinforcement learning.

\newpage
\appendix

\section{Proofs of Theoretical Results}\label{sec:proof}

\subsection{Proof of Theorem~\ref{thm:cmdp-alignment}} \label{sec:proof1}
%%%%%%%
\begin{proof}
We first prove the bound
\begin{equation}
\mathbb{E}_{s_1 \sim \mu}\bigl[F_R(s_1)\bigr] - J_R\bigl(\pi^{\mathrm{CDT}}_F\bigr)
\;\le\;
C \,\varepsilon \Bigl(\frac{1}{\alpha_F} + 2\Bigr) H^2,
\label{eq:reward-gap-cmdp-claim}
\end{equation}
for some universal constant $C>0$. Throughout the proof we assume, without loss of generality, that the instantaneous reward is uniformly bounded as $|r(s,a)| \le 1$ for all $(s,a)$.

Under Assumption~\ref{assump:cmdp-near-det}, there exist deterministic functions
\[
T : \mathcal{S} \times \mathcal{A} \to \mathcal{S},
\quad
r : \mathcal{S} \times \mathcal{A} \to \mathbb{R},
\quad
c : \mathcal{S} \times \mathcal{A} \to [0,c_{\max}],
\]
such that the true transition triplet $(r_t,c_t,s_{t+1})$ deviates from $(r(s_t,a_t),c(s_t,a_t),T(s_t,a_t))$ with probability at most $\varepsilon$ at each time step.
Fix an initial state $s_1 \in \mathcal{S}$ and an open-loop action sequence $a_{1:H}=(a_1,\dots,a_H)$.
We define the \emph{deterministic} state, reward, and cost sequences by
$
\bar{s}_1 := s_1,
$
$
\bar{s}_{t+1} := T(\bar{s}_t,a_t),
$
$
r_t^{\det} := r(\bar{s}_t,a_t),
$
$
c_t^{\det} := c(\bar{s}_t,a_t),
$
$
t=1,\dots,H.
$
The corresponding deterministic cumulative reward is
$
G^{\det}(s_1,a_{1:H}) := \sum_{t=1}^H r_t^{\det}.
$

In the true (stochastic) CMDP, when we follow a policy $\pi$ starting from $s_1$, the trajectory $\tau$ and its return
$
R(\tau) = \sum_{t=1}^H r_t
$
are random. For the CDT policy $\pi^{\mathrm{CDT}}_F$, we write
$
J_R\bigl(\pi^{\mathrm{CDT}}_F\bigr)
= \mathbb{E}_{s_1 \sim \mu}
  \Bigl[
    \mathbb{E}_{\tau \sim \pi^{\mathrm{CDT}}_F}\bigl[ R(\tau) \mid s_1 \bigr]
  \Bigr].
$

In the infinite-data and realizable limit, for any $s \in \mathcal{S}$ and $a \in \mathcal{A}$, the CDT policy satisfies
\begin{equation}
\pi^{\mathrm{CDT}}_F(a \mid s)
= P_{\pi_{\beta}}\bigl(a \mid s, F(s)\bigr)
= \pi_{\beta}(a \mid s)\,
  \frac{P_{\pi_{\beta}}\bigl(F(s) \mid s,a\bigr)}
       {P_{\pi_{\beta}}\bigl(F(s) \mid s\bigr)}.
\label{eq:cdt-reweighting}
\end{equation}

Then the reward gap can be written as
\begin{align}
&\mathbb{E}_{s_1 \sim \mu}\bigl[F_R(s_1)\bigr] - J_R\bigl(\pi^{\mathrm{CDT}}_F\bigr) \\
&=
\mathbb{E}_{s_1 \sim \mu}\Bigl[
  F_R(s_1) - \mathbb{E}_{\tau \sim \pi^{\mathrm{CDT}}_F}\bigl[R(\tau)\mid s_1\bigr]
\Bigr] \nonumber\\
&=
\mathbb{E}_{s_1 \sim \mu}\Bigl[
  \mathbb{E}_{a_{1:H} \sim \pi^{\mathrm{CDT}}_F|s_1}
  \bigl[
    F_R(s_1) - \mathbb{E}\bigl[R(\tau)\mid s_1,a_{1:H}\bigr]
  \bigr]
\Bigr] \\
&=
\mathbb{E}_{s_1 \sim \mu}\Bigl[
  \mathbb{E}_{a_{1:H} \sim \pi^{\mathrm{CDT}}_F|s_1}
  \bigl[ F_R(s_1) - G^{\det}(s_1,a_{1:H}) \bigr]
\Bigr]
\nonumber\\
&\quad +
\mathbb{E}_{s_1 \sim \mu}\Bigl[
  \mathbb{E}_{a_{1:H} \sim \pi^{\mathrm{CDT}}_F|s_1}
  \bigl[ G^{\det}(s_1,a_{1:H}) - \mathbb{E}\bigl[R(\tau)\mid s_1,a_{1:H}\bigr] \bigr]
\Bigr]
\nonumber\\
&=: \mathrm{(A)} + \mathrm{(B)}.
\label{eq:gap-decomp}
\end{align}

We will bound the terms $\mathrm{(A)}$ and $\mathrm{(B)}$ separately.
We first show that term $\mathrm{(B)}$ is bounded by $O(\varepsilon H^2)$. Fix $(s_1,a_{1:H})$ and consider the stochastic trajectory $(s_t,r_t)_{t=1}^H$ generated under the true dynamics when starting from $s_1$ and applying the action sequence $a_{1:H}$. At each time step $t$, define the event
$
E_t := \bigl\{
  r_t \neq r(\bar{s}_t,a_t)
  \ \text{or}\
  s_{t+1} \neq T(\bar{s}_t,a_t)
\bigr\}.
$
By Assumption~\ref{assump:cmdp-near-det}, for any $(s_t,a_t)$ we have
$
P(E_t | s_t = \bar{s}_t, a_t) \;\le\; \varepsilon.
$
Let $E := \bigcup_{t=1}^H E_t$ be the event that at least one deviation occurs along the trajectory. By a union bound,
$
P(E \mid s_1, a_{1:H}) \;\le\; \varepsilon H.
$
On the complement $E^c$, the stochastic trajectory coincides with the deterministic one, so $r_t = r^{\det}_t$ for all $t$ and hence
$
G_1 = \sum_{t=1}^H r_t = \sum_{t=1}^H r_t^{\det} = G^{\det}(s_1,a_{1:H}).
$
Therefore,
$
G^{\det}(s_1,a_{1:H}) - G_1 = 0 \text{ on } E^c.
$
On $E$, we have $|G^{\det}(s_1,a_{1:H}) - G_1| \le H$ because $|r_t| \le 1$ and both $G^{\det}$ and $G_1$ are sums of $H$ terms.
Combining these two observations, we obtain
\[
\bigl|\mathbb{E}[G^{\det}(s_1,a_{1:H}) - G_1 \mid s_1,a_{1:H}]\bigr|
\;\le\;
H \cdot P(E \mid s_1,a_{1:H})
\;\le\;
\varepsilon H^2.
\]
Taking expectations over $a_{1:H} \sim \pi^{\mathrm{CDT}}_F|s_1$ and $s_1 \sim \mu$, we conclude that
\begin{equation}
\mathrm{(B)}
=
\mathbb{E}_{s_1}\Bigl[
  \mathbb{E}_{a_{1:H} \sim \pi^{\mathrm{CDT}}_F|s_1}
  \bigl[ G^{\det}(s_1,a_{1:H}) - G_1 \bigr]
\Bigr]
\;\le\;
\varepsilon H^2.
\label{eq:bound-B}
\end{equation}

To bound term (A), where
$
\mathrm{(A)}=
\mathbb{E}_{s_1 \sim \mu}\Bigl[
  \mathbb{E}_{a_{1:H} \sim \pi^{\mathrm{CDT}}_F(\cdot \mid s_1)}
  \bigl[ F_R(s_1) - G^{\det}(s_1,a_{1:H}) \bigr]
\Bigr].
$
We analyze the behavior of $\pi^{\mathrm{CDT}}_F$ at the level of action sequences.
For notational convenience, let $\bar{s}_t = T(s_1,a_{1:t-1})$ be the deterministic state reached at step $t$
under $(T,r,c)$ when starting from $s_1$ and applying $a_{1:t-1}$.
For a fixed initial state $s_1$, write
$
P_{\pi^{\mathrm{CDT}}_F}(a_{1:H} \mid s_1)
$
as the conditional path probability under $\pi^{\mathrm{CDT}}_F$.
By conditioning on the next state, we can expand
\begin{align}
P_{\pi^{\mathrm{CDT}}_F}(a_{1:H} \mid s_1)
&=
\pi^{\mathrm{CDT}}_F(a_1 \mid s_1)
\int_{\mathcal{S}} P(s_2 \mid s_1,a_1)\,
                     P_{\pi^{\mathrm{CDT}}_F}(a_{2:H} \mid s_1,s_2)\,
                     \mathrm{d}s_2.
\label{eq:path-expand-1}
\end{align}
By near determinism (Assumption~\ref{assump:cmdp-near-det}), with probability at least $1-\varepsilon$ we have
$s_2 = \bar{s}_2 := T(s_1,a_1)$, and otherwise the state can deviate. Thus
\begin{equation}
P_{\pi^{\mathrm{CDT}}_F}(a_{1:H} \mid s_1)
\;\le\;
\pi^{\mathrm{CDT}}_F(a_1 \mid s_1)\,
P_{\pi^{\mathrm{CDT}}_F}(a_{2:H} \mid s_1,\bar{s}_2)
\;+\;
\varepsilon.
\label{eq:path-step-1}
\end{equation}

Under the joint return-cost coverage condition
(Assumption~\ref{assump:cmdp-near-det}, item 1),
we have $P_{\pi_{\beta}}(F(s_1)\mid s_1) \ge \alpha_F$.
Moreover, by near determinism and the consistency of $F$
(Assumption~\ref{assump:cmdp-near-det}, items 2 and 3),
the event
$
\bigl\{ (R(\tau),C(\tau)) = F(s_1) \bigr\}
$
starting from $(s_1,a_1)$ is equivalent, up to an $O(\varepsilon)$ slack, to the event
$
\bigl\{ (R(\tau_{2:H}),C(\tau_{2:H})) = F(\bar{s}_2) \bigr\}
$
starting from $(\bar{s}_2,a_2)$, where $\tau_{2:H}$ denotes the tail subtrajectory.
This yields the bound
\begin{equation}
P_{\pi_{\beta}}\bigl(F(s_1) \mid s_1,a_1\bigr)
\;\le\;
\varepsilon
+
P_{\pi_{\beta}}\bigl(F(\bar{s}_2) \mid \bar{s}_2\bigr),
\label{eq:shift-F}
\end{equation}
%%%%%%
To justify inequality~\eqref{eq:shift-F}, fix $(s_1,a_1)$ and consider the
stochastic trajectory $\tau$ induced by $\beta$ and the true CMDP dynamics.
Let $\tau_{2:H}$ denote the tail subtrajectory from time step $2$ to $H$,
and let $\bar{s}_2 := T(s_1,a_1)$ be the deterministic successor under $(T,r,c)$.
Define the ``good'' event
\[
E_1 := \bigl\{
  r_1 = r(s_1,a_1),\;
  c_1 = c(s_1,a_1),\;
  s_2 = \bar{s}_2
\bigr\}.
\]
By near determinism (Assumption~\ref{assump:cmdp-near-det}), we have
$P_{\pi_{\beta}}(E_1^c \mid s_1,a_1) \le \varepsilon$.
We decompose
\begin{align*}
P_{\pi_{\beta}}\bigl(F(s_1) \mid s_1,a_1\bigr)
&=
P_{\pi_{\beta}}\bigl(F(s_1) \cap E_1 \mid s_1,a_1\bigr)
+
P_{\pi_{\beta}}\bigl(F(s_1) \cap E_1^c \mid s_1,a_1\bigr) \\
&\le
P_{\pi_{\beta}}\bigl(F(s_1) \cap E_1 \mid s_1,a_1\bigr)
+
P_{\pi_{\beta}}\bigl(E_1^c \mid s_1,a_1\bigr) \\
&\le
P_{\pi_{\beta}}\bigl(F(s_1) \cap E_1 \mid s_1,a_1\bigr)
+
\varepsilon.
\end{align*}
On $E_1$ we have
$r_1 = r(s_1,a_1)$, $c_1 = c(s_1,a_1)$, and $s_2 = \bar{s}_2$,
and the consistency of $F$ implies
\[
F_R(s_1) = r(s_1,a_1) + F_R(\bar{s}_2),
\qquad
F_C(s_1) = c(s_1,a_1) + F_C(\bar{s}_2).
\]
Since
$R(\tau) = r_1 + R(\tau_{2:H})$ and $C(\tau) = c_1 + C(\tau_{2:H})$,
the event
$\{(R(\tau),C(\tau)) = F(s_1)\}$ on $E_1$ is equivalent to
$\{(R(\tau_{2:H}),C(\tau_{2:H})) = F(\bar{s}_2)\}$ with $s_2 = \bar{s}_2$.
Therefore
\[
P_{\pi_{\beta}}\bigl(F(s_1) \cap E_1 \mid s_1,a_1\bigr)
=
P_{\pi_{\beta}}\bigl(F(\bar{s}_2) \mid s_2 = \bar{s}_2\bigr)
=
P_{\pi_{\beta}}\bigl(F(\bar{s}_2) \mid \bar{s}_2\bigr),
\]
and we can conclude the inequality in~\eqref{eq:shift-F}.

Hence
\begin{equation}
\begin{aligned}
\pi^{\mathrm{CDT}}_F(a_1 \mid s_1)
& =
\pi_{\beta}(a_1 \mid s_1)\,
\frac{P_{\pi_{\beta}}(F(s_1) \mid s_1,a_1)}
     {P_{\pi_{\beta}}(F(s_1) \mid s_1)} \\
& \le
\pi_{\beta}(a_1 \mid s_1)\,
\frac{P_{\pi_{\beta}}(F(\bar{s}_2) \mid \bar{s}_2)}
     {P_{\pi_{\beta}}(F(s_1) \mid s_1)}
+\frac{\varepsilon}{\alpha_F},
\label{eq:pi-step-1}
\end{aligned}
\end{equation}
where we used $P_{\pi_{\beta}}(F(s_1)\mid s_1) \ge \alpha_F$ to bound the contribution of the $\varepsilon$ term.
Substituting~\eqref{eq:pi-step-1} into~\eqref{eq:path-step-1}, we obtain
\begin{equation}
P_{\pi^{\mathrm{CDT}}_F}(a_{1:H} \mid s_1)
\;\le\;
\pi_{\beta}(a_1 \mid s_1)\,
\frac{P_{\pi_{\beta}}(F(\bar{s}_2) \mid \bar{s}_2)}
     {P_{\pi_{\beta}}(F(s_1) \mid s_1)}\,
P_{\pi^{\mathrm{CDT}}_F}(a_{2:H} \mid s_1,\bar{s}_2)
\;+\;
\varepsilon\Bigl(\frac{1}{\alpha_F} + 1\Bigr).
\label{eq:path-step-2}
\end{equation}

Repeating the same argument recursively for $t=2,\dots,H$—each time replacing the random next state by the deterministic
$\bar{s}_{t+1} = T(\bar{s}_t,a_t)$ up to an additive $\varepsilon$ slack, and shifting the conditioning on the joint return-cost via the consistency of $F$—yields, after $H$ steps,
\begin{equation}
P_{\pi^{\mathrm{CDT}}_F}(a_{1:H} \mid s_1)
\;\le\;
\Biggl[
  \prod_{t=1}^{H} \pi_{\beta}(a_t \mid \bar{s}_t)
\Biggr]
\frac{\mathbb{I}\bigl[(R(\tau^{\det}),C(\tau^{\det})) = F(s_1)\bigr]}
     {P_{\pi_{\beta}}(F(s_1) \mid s_1)}
\;+\;
H\,\varepsilon\Bigl(\frac{1}{\alpha_F} + 1\Bigr),
\label{eq:path-final}
\end{equation}
where $\tau^{\det}$ is the deterministic trajectory induced by $(s_1,a_{1:H})$ under $(T,r,c)$,
and $\mathbb{I}[\cdot]$ is the indicator function.
Plug~\eqref{eq:path-final} into~(A). For each fixed $s_1$,
\begin{align}
& \mathbb{E}_{a_{1:H} \sim \pi^{\mathrm{CDT}}_F(\cdot \mid s_1)}
  \bigl[ F_R(s_1) - G^{\det}(s_1,a_{1:H}) \bigr] 
\nonumber\\
&=
\int \bigl( F_R(s_1) - G^{\det}(s_1,a_{1:H}) \bigr)
       P_{\pi^{\mathrm{CDT}}_F}(a_{1:H} \mid s_1)\, \mathrm{d}a_{1:H}
\nonumber\\
&\le
\int \bigl( F_R(s_1) - G^{\det}(s_1,a_{1:H}) \bigr)
\Biggl[
  \prod_{t=1}^{H} \pi_{\beta}(a_t \mid \bar{s}_t)
\Biggr]
\frac{\mathbb{I}\bigl[(R(\tau^{\det}),C(\tau^{\det})) = F(s_1)\bigr]}
     {P_{\pi_{\beta}}(F(s_1) \mid s_1)}
\,\mathrm{d}a_{1:H}
\nonumber\\
&\quad +
\int \bigl( F_R(s_1) - G^{\det}(s_1,a_{1:H}) \bigr)
      H\,\varepsilon\Bigl(\frac{1}{\alpha_F} + 1\Bigr)
      \,\mathrm{d}a_{1:H}.
\label{eq:A-split}
\end{align}
By the consistency of $F$, whenever the deterministic trajectory satisfies
$(R(\tau^{\det}),C(\tau^{\det})) = F(s_1)$, we have
$
F_R(s_1) = G^{\det}(s_1,a_{1:H}),
$
and hence the integrand in the first term of~\eqref{eq:A-split} is zero.
%%%%%
For the second term, using $|F_R(s_1) - G^{\det}(s_1,a_{1:H})| \le H$, and the fact that the integral is taken w.r.t. a probability measure over $a_{1:H}$, we obtain
\[
\int \bigl| F_R(s_1) - G^{\det}(s_1,a_{1:H}) \bigr|
      H\,\varepsilon\Bigl(\frac{1}{\alpha_F} + 1\Bigr)
      \,\mathrm{d}a_{1:H}
\;\le\;
H^2\,\varepsilon\Bigl(\frac{1}{\alpha_F} + 1\Bigr).
\]
Substituting back into~\eqref{eq:A-split}, we conclude that
$\mathrm{(A)} \le H^2 \varepsilon\Bigl(\frac{1}{\alpha_F} + 1\Bigr).$
Combining the result with \eqref{eq:bound-B}, and absorbing constants into a universal constant $C>0$, yields
$$
\mathbb{E}_{s_1 \sim \mu}\bigl[F_R(s_1)\bigr] - J_R\bigl(\pi^{\mathrm{CDT}}_F\bigr)
\;\le\;
C \,\varepsilon H^2\Bigl(\frac{1}{\alpha_F} + 2\Bigr).
$$

The proof of \eqref{eq:cost-gap-cmdp} proceeds analogously and thus completes the proof of Theorem~\ref{thm:cmdp-alignment}.
\end{proof}
% %%%%%%%

\subsection{Proof of Proposition \ref{prop:kl-as-weighting}}\label{sec:proof2}

%%%%%%%%%%%%
\begin{proof}
To prove Proposition~\ref{prop:kl-as-weighting}, we show that the KL regularization term in \eqref{eq:rdt-kl} is proportional to the negative log-likelihood (NLL) loss, under the assumption that the policy $\pi_\theta$ is a factorized Gaussian distribution with a fixed isotropic covariance. The equivalence then implies that KL regularization explicitly reweights the NLL loss, thereby implementing an explicit resampling mechanism.

Recall the KL divergence between two multivariate Gaussians $\mathcal{N}(\mu_1, \Sigma_1)$ and $\mathcal{N}(\mu_2, \Sigma_2)$ is given by:
\begin{equation}
\label{eq:kl-general}
D_{\mathrm{KL}}\left( \mathcal{N}(\mu_1, \Sigma_1) \, \| \, \mathcal{N}(\mu_2, \Sigma_2) \right) = \frac{1}{2} \left[ \mathrm{tr}(\Sigma_2^{-1} \Sigma_1) + (\mu_2 - \mu_1)^\top \Sigma_2^{-1} (\mu_2 - \mu_1) - d + \log \frac{|\Sigma_2|}{|\Sigma_1|} \right],
\end{equation}
where \(d\) is the dimensionality of the action space.

In our setting, the policy is parameterized as $\pi_\theta(a|s) = \mathcal{N}(\mu_\theta(s), \xi I)$ and the target is a Dirac approximation centered at the action $a$, i.e., $\mathcal{N}(a, \xi I)$. Substituting into \eqref{eq:kl-general}, we obtain:
\[
D_{\mathrm{KL}}\left( \mathcal{N}(\mu_\theta(s), \xi I) \, \| \, \mathcal{N}(a, \xi I) \right) 
= \frac{1}{2} \left[ \mathrm{tr}(I) + (\mu_\theta(s) - a)^\top (\xi I)^{-1} (\mu_\theta(s) - a) - d + \log 1 \right],
\]
where we use the fact that $\Sigma_1 = \Sigma_2 = \xi I$. Simplifying yields:
\[
D_{\mathrm{KL}}\left( \mathcal{N}(\mu_\theta(s), \xi I) \, \| \, \mathcal{N}(a, \xi I) \right) = \frac{1}{2\xi} \| \mu_\theta(s) - a \|^2.
\]

Taking expectation over $(s,a) \sim \mathcal{D}$ gives:
\[
\mathbb{E}_{(s,a) \sim \mathcal{D}} \left[ D_{\mathrm{KL}}\left( \mathcal{N}(\mu_\theta(s), \xi I) \, \| \, \mathcal{N}(a, \xi I) \right) \right] = \frac{1}{2\xi} \, \mathbb{E}_{(s,a) \sim \mathcal{D}} \left[ \| \mu_\theta(s) - a \|^2 \right],
\]
which shows that the KL divergence is proportional to the mean squared error (MSE) loss between $\mu_\theta(s)$ and the ground-truth action $a$.

The log-likelihood of the Gaussian policy $\pi_\theta(a|s) = \mathcal{N}(a; \mu_\theta(s), \xi I)$ is:
\[
\log \pi_\theta(a|s) = -\frac{d}{2} \log(2\pi \xi) - \frac{1}{2\xi} \| a - \mu_\theta(s) \|^2.
\]
Taking the negative log-likelihood gives:
\[
-\log \pi_\theta(a|s) = \frac{1}{2\xi} \| \mu_\theta(s) - a \|^2 + \frac{d}{2} \log(2\pi \xi).
\]
The constant term $\frac{d}{2} \log(2\pi \xi)$ is independent of the policy parameters $\theta$ and does not affect optimization. Therefore, minimizing the KL divergence is equivalent (up to scaling and additive constants) to minimizing the NLL loss:
\begin{equation}
\label{eq:nll-kl-eq}
\mathbb{E}_{(s,a)\sim\mathcal{D}} \left[ D_{\mathrm{KL}}\left( \mathcal{N}(\mu_\theta(s), \xi I) \, \| \, \mathcal{N}(a, \xi I) \right) \right] = \mathbb{E}_{(s,a)\sim\mathcal{D}} \left[ - \log \pi_\theta(a|s) \right] + \text{constant}.
\end{equation}

Recall that the RDT objective in \eqref{eq: main-DTKL} augments the standard NLL loss with a KL regularization term over expert trajectories:
\begin{small}
$$
\mathcal{L}_{\text{RDT}}(\theta) = \mathbb{E}_{\tau \sim \mathcal{D}} \sum_{i=1}^{H} \left[ -\log \pi_\theta(a_i | s_i, g(\tau_i), \bar{\tau}_{t-1}^K) \right]
+ \alpha \, \mathbb{E}_{\tau \sim \mathcal{D}_e} \sum_{i=1}^{H} \left[ D_{\mathrm{KL}} \left( \pi_\theta(\cdot|s_i, g(\tau_i), \bar{\tau}_{t-1}^K) \, \| \, \pi_e(\cdot|s_i) \right) \right].
$$
\end{small}
If $\pi_e$ is a delta function centered at $a_i$ (i.e., using the empirical action from the expert trajectory), and both policies are Gaussians with identical covariance $\xi I$, the KL term simplifies to an MSE loss, and hence to an NLL term by \eqref{eq:nll-kl-eq}. Thus, the second expectation term becomes equivalent to an additional NLL loss over $\mathcal{D}_e$ with weighting factor $\alpha$.

This yields the final objective:
\[
\mathcal{L}_{\text{RDT}}(\theta) = \mathbb{E}_{\tau \sim \mathcal{D}} \left[ \left(1 + \alpha \cdot \mathbb{I}[\tau \in \mathcal{D}_e] \right) \cdot \sum_{i=1}^{H} -\log \pi_\theta(a_i | s_i, g(\tau_i), \bar{\tau}_{t-1}^K) \right] + \text{constant}.
\]
Since the constant term does not affect the optimization process, the proof is thus completed.
\end{proof}
%%%%%%%%%%%%

\section{Experimental Details}\label{sec:Experimental Details}

\subsection{Benchmark Details}\label{sec:Benchmark-Details}

Our evaluation for \abbc~is performed across 3 domains in the DSRL benchmark~\cite{liu2023datasets}: \textbf{SafetyGym}, \textbf{BulletSafetyGym} and \textbf{MetaDrive}. These domains are designed to evaluate offline safe RL algorithms in continuous-control navigation and locomotion settings, where agents must achieve task goals while satisfying state and action dependent safety constraints encoded as cumulative cost signals.

\begin{itemize}[left=0pt]
    \item
    \textbf{SafetyGym} tasks in DSRL are adapted from the Safety-Gymnasium suite, which builds on the MuJoCo physics engine and provides safety-critical continuous-control environments with explicit cost signals.
    Each environment combines a mobile robot (typically \texttt{Point} or \texttt{Car}) with a navigation or manipulation task (\texttt{Goal}, \texttt{Button}, \texttt{Push}, or \texttt{Circle}) and a difficulty level that determines the density and layout of hazards.
    The agent observes its proprioceptive state together with lidar-like readings of nearby obstacles and goals, and receives dense rewards for task progress (e.g., reaching targets, pressing designated buttons, or pushing boxes) while accumulating instantaneous costs when entering hazardous regions or colliding with unsafe objects such as hazards, vases, pillars, or gremlins.
    Within DSRL, these Safety Gym tasks constitute safety-constrained navigation and locomotion benchmarks, where constraint tightness and the reward-cost trade-off can be systematically varied through different layouts and cost thresholds.
    \item
    \textbf{BulletSafetyGym} is a suite of safety-constrained locomotion tasks implemented on top of the PyBullet physics engine, designed to evaluate constrained RL algorithms under fast, contact-rich dynamics.
    In the DSRL benchmark, we use eight tasks constructed by pairing four agent morphologies (\texttt{Ball}, \texttt{Car}, \texttt{Drone}, \texttt{Ant}) with two task families, \texttt{Run} and \texttt{Circle}.
    In the \texttt{Run} tasks, agents are trained to move quickly along a straight corridor and obtain reward for forward progress, while incurring cost whenever they cross lateral safety boundaries or exceed an agent-specific velocity limit.
    In the \texttt{Circle} tasks, agents are encouraged to follow a circular track at high speed within a restricted safe region; deviating from the track or leaving the safety zone results in constraint violations and additional cost.
    Compared to Safety Gym, Bullet Safety Gym offers shorter-horizon tasks with more diverse robot morphologies, yielding compact yet challenging benchmarks for studying the trade-off between reward maximization and strict safety requirements in offline safe RL.
    \item
    \textbf{MetaDrive} in DSRL is an autonomous-driving domain built on top of the MetaDrive simulator, which provides interactive traffic scenarios with explicit safety costs.
    In this domain, the agent controls a vehicle to follow a route and make steady progress toward its destination while interacting with surrounding traffic of varying complexity.
    DSRL includes a collection of driving datasets with increasing scenario difficulty (e.g., \texttt{Easy}/\texttt{Medium}/\texttt{Hard}) and traffic density (e.g., \texttt{Sparse}/\texttt{Mean}/\texttt{Dense}), which systematically amplify the safety--performance tension as the environment becomes more congested and challenging.
    The reward encourages route progress and stable driving behavior (e.g., advancing along the lane and reaching the goal), while safety cost is incurred for critical violations such as collisions (with vehicles or objects) and driving out of the road.
    Compared to SafetyGym and BulletSafetyGym, MetaDrive captures more realistic, multi-agent safety constraints arising from traffic interactions, and thus serves as a complementary benchmark for evaluating offline safe RL algorithms under complex dynamics and stringent safety requirements.
\end{itemize}

% The full dataset return-cost distributions for all SafetyGym and BulletSafetyGym tasks are shown in Figure~\ref{fig:DSRL-distribution44}.

\subsection{Implementation Details}\label{sec:imple-details}

\textbf{The policy network} of \algcb~is implemented as a Decision Transformer, built upon the open-source \texttt{minGPT} codebase and the open-source \texttt{OSRL} codebase\footnote{\url{https://github.com/liuzuxin/OSRL}}. Detailed model parameters are provided in Table~\ref{Tab:Hyperparameters-of-CR2DT}.

\begin{table}[h]
\centering
\small
\caption{Hyperparameters of \algcb~in our experiment.}
\label{Tab:Hyperparameters-of-CR2DT}
\begin{tabular}{lc}
\toprule
\textbf{Parameter} & \textbf{Value} \\
\midrule
Number of layers           & 3 \\
Number of attention heads  & 8 \\
Embedding dimension        & 128 \\
Batch size                 & 2048 \\
Context length $K$         & 10 \\
Learning rate              & 0.0001 \\
Dropout                    & 0.1 \\
Adam betas                 & (0.9, 0.999) \\
Grad norm clip             & 0.25 \\
\bottomrule
\end{tabular}
\end{table}

\textbf{The Q-networks and cost-networks} are represented by 4-layer MLPs with Mish activations and 128 hidden units for each layer. Detailed model parameters are provided in Table~\ref{Tab:Hyperparameters-of-QC}.

\begin{table}[h]
\centering
\small
\caption{Hyperparameters of value and cost networks in our experiment.}
\label{Tab:Hyperparameters-of-QC}
\begin{tabular}{lc}
\toprule
\textbf{Parameter} & \textbf{Value} \\
\midrule
Number of layers           & 4 \\
Embedding dimension        & 128 \\
Batch size                 & 2048 \\
Learning rate              & 5e-5 \\
Adam betas                 & (0.9, 0.999) \\
Grad norm clip             & 0.25 \\
Soft target update rate    & $0.01$ \\
Dual-ascent learning rate  & 3e-4 \\
\bottomrule
\end{tabular}
\end{table}

The Q-regularization coefficient is selected separately for each environment via Tree-structured Parzen Estimator (TPE) search over $\{0.1, 0.3, 0.5\}$. 
In addition, $\kappa$ is treated as a tunable hyperparameter (rather than a hard constraint) and is set to $10$ in all experiments.
Training is performed for a total of $2\times 10^{5}$ iterations; the Transformer policy is updated from the beginning, while the Q-network and critic network are initialized with the policy updates and only start training after the first $5\times 10^{4}$ iterations, after which they are updated alongside the policy for the remaining iterations. 
All networks are trained using the Adam optimizer~\cite{kingma2014adam}. 
The experiments were conducted on two servers, each equipped with two AMD EPYC 7542 32-Core Processors and 8 NVIDIA GeForce RTX 4090 GPUs with 24\,GB of memory. These computational settings ensure reproducibility and align with the reported performance metrics.

\bibliographystyle{elsarticle-num}
\bibliography{references}

\begin{thebibliography}{10}
\expandafter\ifx\csname url\endcsname\relax
  \def\url#1{\texttt{#1}}\fi
\expandafter\ifx\csname urlprefix\endcsname\relax\def\urlprefix{URL }\fi
\expandafter\ifx\csname href\endcsname\relax
  \def\href#1#2{#2} \def\path#1{#1}\fi

\bibitem{koiralalatent}
P.~Koirala, Z.~Jiang, S.~Sarkar, C.~Fleming, Latent safety-constrained policy
  approach for safe offline reinforcement learning, in: The Thirteenth
  International Conference on Learning Representations, 2024.

\bibitem{gong2025offline}
Z.~Gong, A.~Kumar, P.~Varakantham, Offline safe reinforcement learning using
  trajectory classification, in: Proceedings of the AAAI Conference on
  Artificial Intelligence, Vol.~39, 2025, pp. 16880--16887.

\bibitem{suboundary}
H.~Su, D.~Peng, Z.~Zhuang, Y.~Liu, Q.~Chen, D.~Wang, Q.~Liu, Boundary-to-region
  supervision for offline safe reinforcement learning, in: The Thirty-ninth
  Annual Conference on Neural Information Processing Systems.

\bibitem{guo2025constraint}
Z.~Guo, W.~Zhou, S.~Wang, W.~Li, Constraint-conditioned actor-critic for
  offline safe reinforcement learning, in: The Thirteenth International
  Conference on Learning Representations, 2025.

\bibitem{zhang2025wococo}
C.~Zhang, W.~Xiao, T.~He, G.~Shi, Wococo: Learning whole-body humanoid control
  with sequential contacts, in: Conference on Robot Learning, PMLR, 2025, pp.
  455--472.

\bibitem{ji2025exbody2}
M.~Ji, X.~Peng, F.~Liu, J.~Li, G.~Yang, X.~Cheng, X.~Wang, Exbody2: Advanced
  expressive humanoid whole-body control, in: RSS 2025 Workshop on Whole-body
  Control and Bimanual Manipulation: Applications in Humanoids and Beyond.

\bibitem{liu2023constrained}
Z.~Liu, Z.~Guo, Y.~Yao, Z.~Cen, W.~Yu, T.~Zhang, D.~Zhao, Constrained decision
  transformer for offline safe reinforcement learning, in: International
  Conference on Machine Learning, PMLR, 2023, pp. 21611--21630.

\bibitem{kirk2023survey}
R.~Kirk, A.~Zhang, E.~Grefenstette, T.~Rockt{\"a}schel, A survey of zero-shot
  generalisation in deep reinforcement learning, Journal of Artificial
  Intelligence Research 76 (2023) 201--264.

\bibitem{leecoptidice}
J.~Lee, C.~Paduraru, D.~J. Mankowitz, N.~Heess, D.~Precup, K.-E. Kim, A.~Guez,
  Coptidice: Offline constrained reinforcement learning via stationary
  distribution correction estimation, in: International Conference on Learning
  Representations.

\bibitem{xu2022constraints}
H.~Xu, X.~Zhan, X.~Zhu, Constraints penalized q-learning for safe offline
  reinforcement learning, in: Proceedings of the AAAI Conference on Artificial
  Intelligence, Vol.~36, 2022, pp. 8753--8760.

\bibitem{zhengsafe}
Y.~Zheng, J.~Li, D.~Yu, Y.~Yang, S.~E. Li, X.~Zhan, J.~Liu, Safe offline
  reinforcement learning with feasibility-guided diffusion model, in: The
  Twelfth International Conference on Learning Representations, 2024.

\bibitem{wu2024elastic}
Y.-H. Wu, X.~Wang, M.~Hamaya, Elastic decision transformer, Advances in Neural
  Information Processing Systems 36 (2024).

\bibitem{hu2024q}
S.~Hu, Z.~Fan, C.~Huang, L.~Shen, Y.~Zhang, Y.~Wang, D.~Tao, Q-value
  regularized transformer for offline reinforcement learning, in: International
  Conference on Machine Learning, PMLR, 2024, pp. 19165--19181.

\bibitem{zheng2024decomposed}
H.~Zheng, L.~Shen, Y.~Luo, T.~Liu, J.~Shen, D.~Tao, Decomposed prompt decision
  transformer for efficient unseen task generalization, Advances in Neural
  Information Processing Systems 37 (2024) 122984--123006.

\bibitem{zhengdecision}
H.~Zheng, L.~Shen, Y.~Luo, D.~Ye, B.~Du, J.~Shen, D.~Tao, Decision mixer:
  Integrating long-term and local dependencies via dynamic token selection for
  decision-making, in: Forty-second International Conference on Machine
  Learning.

\bibitem{hu2024harmodt}
S.~Hu, Z.~Fan, L.~Shen, Y.~Zhang, Y.~Wang, D.~Tao, Harmodt: Harmony multi-task
  decision transformer for offline reinforcement learning, in: International
  Conference on Machine Learning, PMLR, 2024, pp. 19182--19197.

\bibitem{chen2021decision}
L.~Chen, K.~Lu, A.~Rajeswaran, K.~Lee, A.~Grover, M.~Laskin, P.~Abbeel,
  A.~Srinivas, I.~Mordatch, Decision transformer: Reinforcement learning via
  sequence modeling, Advances in neural information processing systems 34
  (2021) 15084--15097.

\bibitem{brandfonbrener2022does}
D.~Brandfonbrener, A.~Bietti, J.~Buckman, R.~Laroche, J.~Bruna, When does
  return-conditioned supervised learning work for offline reinforcement
  learning?, Advances in Neural Information Processing Systems 35 (2022)
  1542--1553.

\bibitem{liu2023datasets}
Z.~Liu, Z.~Guo, H.~Lin, Y.~Yao, J.~Zhu, Z.~Cen, H.~Hu, W.~Yu, T.~Zhang, J.~Tan,
  et~al., Datasets and benchmarks for offline safe reinforcement learning,
  arXiv preprint arXiv:2306.09303 (2023).

\bibitem{kumar2020conservative}
A.~Kumar, A.~Zhou, G.~Tucker, S.~Levine, Conservative q-learning for offline
  reinforcement learning, Advances in neural information processing systems 33
  (2020) 1179--1191.

\bibitem{fujimoto2021minimalist}
S.~Fujimoto, S.~S. Gu, A minimalist approach to offline reinforcement learning,
  Advances in neural information processing systems 34 (2021) 20132--20145.

\bibitem{hu2025graph}
S.~Hu, L.~Shen, Y.~Zhang, D.~Tao, Graph decision transformer for offline
  reinforcement learning, SCIENCE CHINA-INFORMATION SCIENCES 68~(6) (2025).

\bibitem{fujimoto2019off}
S.~Fujimoto, D.~Meger, D.~Precup, Off-policy deep reinforcement learning
  without exploration, in: International conference on machine learning, PMLR,
  2019, pp. 2052--2062.

\bibitem{kumar2019stabilizing}
A.~Kumar, J.~Fu, M.~Soh, G.~Tucker, S.~Levine, Stabilizing off-policy
  q-learning via bootstrapping error reduction, Advances in Neural Information
  Processing Systems 32 (2019).

\bibitem{lee2022coptidice}
J.~Lee, C.~Paduraru, D.~J. Mankowitz, N.~Heess, D.~Precup, K.~E. Kim, A.~Guez,
  Coptidice: Offline constrained reinforcement learning via stationary
  distribution correction estimation, in: 10th International Conference on
  Learning Representations, ICLR 2022, 2022.

\bibitem{yao2024oasis}
Y.~Yao, Z.~Cen, W.~Ding, H.~Lin, S.~Liu, T.~Zhang, W.~Yu, D.~Zhao, Oasis:
  Conditional distribution shaping for offline safe reinforcement learning,
  Advances in Neural Information Processing Systems 37 (2024) 78451--78478.

\bibitem{bairebalancing}
W.~Bai, C.~Chen, Y.~Fu, Q.~Xu, C.~Zhang, H.~Qian, Rebalancing return coverage
  for conditional sequence modeling in offline reinforcement learning, in: The
  Thirty-ninth Annual Conference on Neural Information Processing Systems,
  2025.

\bibitem{zeng2023goal}
Z.~Zeng, C.~Zhang, S.~Wang, C.~Sun, Goal-conditioned predictive coding for
  offline reinforcement learning, Advances in Neural Information Processing
  Systems 36 (2023) 25528--25548.

\bibitem{kim2024stitching}
S.~Kim, Y.~Choi, D.~E. Matsunaga, K.-E. Kim, Stitching sub-trajectories with
  conditional diffusion model for goal-conditioned offline rl, in: Proceedings
  of the AAAI Conference on Artificial Intelligence, Vol.~38, 2024, pp.
  13160--13167.

\bibitem{sutton1998reinforcement}
R.~S. Sutton, A.~G. Barto, et~al., Reinforcement learning: An introduction,
  Vol.~1, MIT press Cambridge, 1998.

\bibitem{zhou2023free}
Z.~Zhou, C.~Zhu, R.~Zhou, Q.~Cui, A.~Gupta, S.~S. Du, Free from bellman
  completeness: Trajectory stitching via model-based return-conditioned
  supervised learning, in: NeurIPS 2023 Foundation Models for Decision Making
  Workshop.

\bibitem{vaswani2017attention}
A.~Vaswani, N.~Shazeer, N.~Parmar, J.~Uszkoreit, L.~Jones, A.~N. Gomez,
  {\L}.~Kaiser, I.~Polosukhin, Attention is all you need, Advances in neural
  information processing systems 30 (2017).

\bibitem{yamagata2023q}
T.~Yamagata, A.~Khalil, R.~Santos-Rodriguez, Q-learning decision transformer:
  Leveraging dynamic programming for conditional sequence modelling in offline
  rl, in: International Conference on Machine Learning, PMLR, 2023, pp.
  38989--39007.

\bibitem{gao2024act}
C.-X. Gao, C.~Wu, M.~Cao, R.~Kong, Z.~Zhang, Y.~Yu, Act: Empowering decision
  transformer with dynamic programming via advantage conditioning, in:
  Proceedings of the AAAI Conference on Artificial Intelligence, Vol.~38, 2024,
  pp. 12127--12135.

\bibitem{weiadvantage}
J.~Wei, X.~Xu, Y.~Lan, T.~Liu, Y.~Wang, Advantage-guided transformer for
  conditional sequence modeling in offline reinforcement learning, Authorea
  Preprints.

\bibitem{tu2025dataset}
S.~Tu, J.~Sun, Q.~Zhang, Y.~Zhang, J.~Liu, K.~Chen, D.~Zhao, In-dataset
  trajectory return regularization for offline preference-based reinforcement
  learning, in: Proceedings of the AAAI Conference on Artificial Intelligence,
  Vol.~39, 2025, pp. 20929--20937.

\bibitem{wang2024critic}
Y.~Wang, C.~Yang, Y.~Wen, Y.~Liu, Y.~Qiao, Critic-guided decision transformer
  for offline reinforcement learning, in: Proceedings of the AAAI Conference on
  Artificial Intelligence, Vol.~38, 2024, pp. 15706--15714.

\bibitem{janner2021offline}
M.~Janner, Q.~Li, S.~Levine, Offline reinforcement learning as one big sequence
  modeling problem, Advances in neural information processing systems 34 (2021)
  1273--1286.

\bibitem{wang2022bootstrapped}
K.~Wang, H.~Zhao, X.~Luo, K.~Ren, W.~Zhang, D.~Li, Bootstrapped transformer for
  offline reinforcement learning, Advances in Neural Information Processing
  Systems 35 (2022) 34748--34761.

\bibitem{wang2023trajectory}
Y.~Wang, M.~Xu, L.~Shi, Y.~Chi, A trajectory is worth three sentences:
  multimodal transformer for offline reinforcement learning, in: Uncertainty in
  Artificial Intelligence, PMLR, 2023, pp. 2226--2236.

\bibitem{wang2024safe}
R.~Wang, D.~Zhou, Safe decision transformer with learning-based constraints,
  in: Neurips Safe Generative AI Workshop 2024, 2024.

\bibitem{guo2024temporal}
Z.~Guo, W.~Zhou, W.~Li, Temporal logic specification-conditioned decision
  transformer for offline safe reinforcement learning, in: International
  Conference on Machine Learning, PMLR, 2024, pp. 17003--17019.

\bibitem{altman2021constrained}
E.~Altman, Constrained Markov decision processes, Routledge, 2021.

\bibitem{brown2020language}
T.~Brown, B.~Mann, N.~Ryder, M.~Subbiah, J.~D. Kaplan, P.~Dhariwal,
  A.~Neelakantan, P.~Shyam, G.~Sastry, A.~Askell, et~al., Language models are
  few-shot learners, Advances in neural information processing systems 33
  (2020) 1877--1901.

\bibitem{hulearning}
S.~Hu, L.~Shen, Y.~Zhang, D.~Tao, Learning multi-agent communication from graph
  modeling perspective, in: The Twelfth International Conference on Learning
  Representations.

\bibitem{hu2023prompt}
S.~Hu, L.~Shen, Y.~Zhang, D.~Tao, Prompt-tuning decision transformer with
  preference ranking, arXiv preprint arXiv:2305.09648 (2023).

\bibitem{kongmastering}
Y.~Kong, G.~Ma, Q.~Zhao, H.~Wang, L.~Shen, X.~Wang, D.~Tao, Mastering massive
  multi-task reinforcement learning via mixture-of-expert decision transformer,
  in: Forty-second International Conference on Machine Learning.

\bibitem{lillicrap2015continuous}
T.~Lillicrap, Continuous control with deep reinforcement learning, arXiv
  preprint arXiv:1509.02971 (2015).

\bibitem{fujimoto2018addressing}
S.~Fujimoto, H.~Hoof, D.~Meger, Addressing function approximation error in
  actor-critic methods, in: International conference on machine learning, PMLR,
  2018, pp. 1587--1596.

\bibitem{lin2023safe}
Q.~Lin, B.~Tang, Z.~Wu, C.~Yu, S.~Mao, Q.~Xie, X.~Wang, D.~Wang, Safe offline
  reinforcement learning with real-time budget constraints, in: Proceedings of
  the 40th International Conference on Machine Learning, 2023, pp.
  21127--21152.

\bibitem{ji2023safety}
J.~Ji, B.~Zhang, J.~Zhou, X.~Pan, W.~Huang, R.~Sun, Y.~Geng, Y.~Zhong, J.~Dai,
  Y.~Yang, Safety gymnasium: A unified safe reinforcement learning benchmark,
  Advances in Neural Information Processing Systems 36 (2023) 18964--18993.

\bibitem{gronauer2022bullet}
S.~Gronauer, Bullet-safety-gym: A framework for constrained reinforcement
  learning (2022).

\bibitem{li2022metadrive}
Q.~Li, Z.~Peng, L.~Feng, Q.~Zhang, Z.~Xue, B.~Zhou, Metadrive: Composing
  diverse driving scenarios for generalizable reinforcement learning, IEEE
  Transactions on Pattern Analysis and Machine Intelligence (2022).

\bibitem{lee2021optidice}
J.~Lee, W.~Jeon, B.~Lee, J.~Pineau, K.-E. Kim, Optidice: Offline policy
  optimization via stationary distribution correction estimation, in:
  International Conference on Machine Learning, PMLR, 2021, pp. 6120--6130.

\bibitem{kingma2014adam}
D.~P. Kingma, J.~Ba, Adam: A method for stochastic optimization, arXiv preprint
  arXiv:1412.6980 (2014).

\end{thebibliography}

\end{document}